\documentclass{article} 
\usepackage{iclr2027_conference,times}
\usepackage{caption}
\usepackage{amsmath,amssymb,amsthm}
\usepackage{graphicx}
\usepackage{booktabs}
\usepackage{microtype}
\usepackage{xcolor}
\usepackage{url}
\usepackage[colorlinks=true,linkcolor=blue!60!black,citecolor=blue!60!black,urlcolor=blue!60!black]{hyperref}
\usepackage{comment}
\usepackage{subcaption}
\newcommand{\J}{J}                       
\newcommand{\cell}{c}                    
\newcommand{\dc}{d(\cell)}               
\newcommand{\Qc}{Q(\cell,a)}             
\newcommand{\pol}{\pi}                   
\newcommand{\code}{\sigma}               
\newcommand{\temp}{\tau}                 
\newcommand{\invtemp}{\beta}             
\newcommand{\betaent}{\tau}              

\newcommand{\Ncells}{N}
\newcommand{\Fmf}{F_{\mathrm{var}}}
\newcommand{\Ftrue}{F}

\newcommand{\Amono}{\mathcal{T}}
\newcommand{\alphabet}{\mathcal{A}}

\newcommand{\nr}[1]{\textcolor{blue}{[NR: #1]}}
\newcommand{\zr}[1]{\textcolor{purple}{[ZR: #1]}}
\newcommand{\q}[1]{\textcolor{red}{[Question: #1]}}

\newtheorem{theorem}{Theorem}[section]
\newtheorem{lemma}[theorem]{Lemma}

\newtheorem{corollary}[theorem]{Corollary}
\newtheorem{statement}{Statement}[section]
\usepackage{enumitem}
\theoremstyle{definition}
\newtheorem{definition}[theorem]{Definition}

\theoremstyle{remark}
\newtheorem{remark}[theorem]{Remark}
\newcommand{\A}{\mathcal{A}}
\newcommand{\X}{\mathcal{X}}
\newcommand{\C}{\mathcal{C}}
\newcommand{\R}{\mathbb{R}}
\newcommand{\Pol}{\Pi}
\newcommand{\simp}{\Delta}
\newcommand{\supp}{\operatorname{supp}}

\newcommand{\step}{\operatorname{st}}
\newcommand{\Nb}{\mathcal{O}}

\usepackage[table]{xcolor}
 \iclrfinalcopy

\title{RLVR landscapes for iterated multiplications can be benign: Insights from spin-glass theory }

\author{Noa Rubin \thanks{Corresponding author.} \\ Racah Institute of Physics \\ The Hebrew University
of Jerusalem\\
\texttt{noa.rubin@mail.huji.ac.il} \\ 
\And 
Zohar Ringel \\
Racah Institute of Physics\\
The Hebrew University of Jerusalem\\
\texttt{zohar.ringel@mail.huji.ac.il}
}

\begin{document}
\maketitle
\lhead{}

\begin{abstract}
Despite the importance of reinforcement learning with verifiable rewards (RLVR), the extent to which it can learn new reasoning capabilities remains debated. Here we study the optimization landscape of RLVR on algorithmic tasks, such as iterated group and quasigroup multiplication. To this end, we map entropy-regularized RLVR over myopic tabular policies onto an energy-based (spin-glass) model over deterministic policies. This mapping upper-bounds what RLVR can achieve, and lets us rigorously characterize the landscape in this tabular setting. We show, both theoretically and experimentally, that for a wide class of models and tasks with uncorrelated inputs, this landscape is benign, containing no local minima that could trap RLVR training. Rather, the practical difficulty of these tasks appears to stem, at least in part, from issues such as diffusive barriers and gradient-estimation error in traversing the landscape. These are genuine obstacles that can prevent a solution from being found, but they are distinct from the landscape itself being rugged. We show that these obstacles can often be mitigated through the choice of entropy regulator. Consistent with this theory, we find that a transformer trained from scratch, using only last-token rewards, successfully learns an algorithmic chain of thought for iterated non-Abelian group multiplications.
\end{abstract}

\section{Introduction}
\label{sec:intro}
Reinforcement learning with verifiable rewards (RLVR), policy-gradient training against an automatically checkable success signal, is the main tool behind recent improvements in the reasoning of large language models
\citep{deepseekr1,dapo}. Whether it produces genuinely new reasoning, or
mainly reweights behavior already present in the pretrained model, is
debated \citep{yue2025,wu2026invisibleleashrlvrescape,wen2025,tsilivis2025reinforcementlearningnexttokenprediction}. 
Behind this debate sit open optimisation questions about RLVR itself. Due to the credit assignment problem, intermediate tokens receive little to no direct signal, and sampled
gradients are noisy \citep{dapo,tao2025hybridreinforcementrewardsparse,han2026nonuniformnoisetosignalratioreinforce,li2026optimaltokenbaselinevariance}.
Such hurdles can be, in principle,  addressed by larger batches, dynamic sampling, and
difficulty curricula \citep{dapo,jiang2025vcrlvariancebasedcurriculumreinforcement}, but these are not generic cures. Curricula, for instance, are subject to primacy bias and shortcut learning
\citep{nikishin2022,liu2023shortcuts}. Indeed, practitioners routinely describe runs as
``trapped in a local optimum'' \citep{ragen,jin2026revisitingentropyreinforcementlearning}. 

This brings to question the complexity and ruggedness of RLVR landscapes. Such analyses of
reinforcement learning exist for specialized settings (see related literature), which show that optimizing policies
with bounded memory is NP-hard in the worst case
\citep{littman1994,mundhenk2000,vlassis2012}. However, worst-case statements say
little about typical instances. For random constraint satisfaction problems, analysis of the typical case was pushed forward by applying statistical mechanics approaches based on spin-glass theory. For instance, spin-glass theory traced typical-case algorithmic hardness to the
geometry of the solution landscape, which above certain thresholds breaks
into many separated clusters and local minima
\citep{mpz2002,krzakala2007,zdeborova2016}, inspiring rigorous mathematical approaches \citep{gamarnik2021}. Building a similar bridge between spin-glass theory and RLVR is thus a desirable goal.

As a toy setting to approach this goal, we focus on a simple class of algorithmic tasks consisting of iterated quasigroup and group multiplication \citep{liu2023shortcuts,merrill2024illusion,li2024serial,grazzi2025statetracking,
marchetti2026sequential}. This choice is partially motivated by a simplistic reading of what pretraining provides: a set of
skills \citep{michaud2024quantizationmodelneuralscaling} the model can execute individually, together with the problem of
composing them in the right order. Taking the skills to be the elements of
a finite group or quasigroup and composition to be its operation,
``learning to reason'' becomes evaluating a product of $T$ elements
presented in sequence, with one intermediate token emitted per element and
reward only if the final token is the correct product. 

As model classes, we use myopic tabular policies (myopic in the sense of \cite{deweese2026kstep}) that read a bounded window of the
input together with their own recent tokens. While such myopic bias is obviously a crude approximation to that of a transformer, the workhorse of RLVR applications in reasoning, it has several merits: It fleshes out culprits in the RLVR loss itself rather than in the policy parametrization. It exhibits some of the qualitative features seen in transformer traps for our tasks, while permitting essentially exact numerics. 

In this work, we provide a novel spin-glass perspective on RLVR, which allows an in-depth analysis of the above task with some conclusions carrying over to transformers. In particular, we establish the following results

\textbf{A mapping between RLVR and spin models}
(Sec. ~\ref{sec:dictionary}). 
We build a dictionary between RLVR and
spin-glass models, yielding three results: every local maximum of the
reward is continuously connected, without crossing a
loss barrier, to a deterministic policy, referred to as a {\it code} or {\it spin-configuration} (\S\ref{sec:codes_reward}); with an
entropy regulator, the RLVR loss is provably lower-bounded by the free
energy of a corresponding spin-glass energy-based model
(\S\ref{sec:free_energy_bound}); and steady states of local dynamics on the spin-glass model's trapped
configurations approximately coincide with  RLVR fixed points
(\S\ref{sec:pure_states_traps}).

\textbf{Spin-glass model for parity task is solvable} (Sec. ~\ref{sec:parity}). We obtain closed expressions for the free energy of the spin-glass model associated with parity at any temperature and depth ($T$) and find a first-order phase transition at large $T$, between strictly random and correct codes.

\textbf{Reward landscapes can be benign} (Sec. ~\ref{sec:escape}). We leverage the spin-glass mapping to prove that for simple untied models the reward has no suboptimal local maxima. We show that a simple Monte Carlo sampler of the spin-glass models solves all untied models and tasks without a curriculum. With weight tying and
input correlations, the proof breaks down, and indeed we find local maxima.

\textbf{The RLVR--spin-glass gap and regulator choices}
(Sec.~\ref{sec:further_results}). We show that RLVR optimization (via exact local policy updates) fails to optimize untied models, while as aforementioned, Monte-Carlo on spin-glass models succeeds. This gap is due to the former's inability to diffuse over equal reward code landscapes. We show that this gap can be closed with local policy resets to random deterministic policies or strong regulators.

\textbf{Strongly regulated transformers solve the task}
(Sec.~\ref{sec:transformer}). We show numerically that transformers can easily express and interpolate between untied tabular policies with narrow windows and hence share some of their benign landscape features. 
 We further show transformer trained from scratch on the
same task fail when left weakly or moderately regulated, but often yield perfect learning of the task when strongly regulated and augmented by a curriculum to circumvent signal-to-noise issues.  

\paragraph{Additional related literature.}
Landscape analyses of
reinforcement learning exist for specialized settings: gradient-domination results for complete parameterisations
\citep{agarwal2019,bhandari2019}, the linear-quadratic regulator
\citep{fazel2018}, convergence for softmax tabular policies \citep{mei2020},
and a study of what entropy regularisation does to small exact-gradient
problems \citep{ahmed2019}. Our work is close in spirit but distinct from the maximum-entropy formulation of control, in which
an entropy-regularised objective is read as inference over a Boltzmann
measure on \emph{trajectories}
\citep{ziebart2008,levine2018,haarnoja2018}: there the sites are the time
steps of a rollout, whereas our ensemble is over policies and the sites
are table cells. On the physics side, a glassy phase was demonstrated for
an optimal-control landscape \citep{day2019glassy}, spin-glass language
has been applied by analogy to supervised loss surfaces
\citep{choromanska2015loss}, and reinforcement learning has been used as a
solver for spin-glass ground states \citep{dirac}. The trapped
configurations we certify have antecedents as partial pooling equilibria
in signalling games \citep{lewis1969,crawford1982,huttegger2010,skyrms2010},
as miscoordination of independent learners in common-payoff games
\citep{panait2005,matignon2012}, and as myopic optima of restricted policy
classes \citep{deweese2026kstep}. Entropy collapse in RLVR and entropy-management
schemes are active empirical topics \citep{cui2025entropy,wang2025eighty,paec};
the diagnostics here differ in being derived from measured landscape
structure rather than fitted to online statistics. 
The stalls of plain RLVR and their repair by
re-randomisation are the multi-step counterparts of partial pooling and its
destabilisation by mutation 
\citep{zhang2004umda} and are kept from premature fixation by margins acting as mutation.


\section{Setting: Task and Models}
\label{sec:setup}
Let $\alphabet$ be a finite alphabet of size $q$ carrying a binary operation
$B:\alphabet\times\alphabet\to\alphabet$ whose multiplication table is a
Latin square, i.e.\ every row and every column is a permutation of
$\alphabet$. Such a structure is a \emph{quasigroup}; a group is the special
case in which $B$ is also associative. The input is a sequence
$x_1\ldots x_T$ drawn uniformly from $\alphabet^T$ (in some expereimental setups we break this assumption and draw correlated inputs), and the target is the
left fold
\begin{equation}
s_1 = x_1,\qquad s_t = B[s_{t-1},x_t],\qquad \text{target} = s_T .
\label{eq:fold}
\end{equation}
We refer to $s_t$ as the \emph{state} of the task at step
$t$. To solve this task, we consider models that sequentially generate an output sequence $y = (y_1, \ldots, y_T)$ conditioned on the input $x$. At each token index $t$, the next token $y_t$ is drawn from a distribution, or policy, $\pi$ that can depend on $x,y_{<t}, t$. An optimal policy is one that perfectly outputs the target, $y_T = s_T$, regardless of intermediate tokens. In this section, we define two model classes for this task: tabular policies trained with RLVR (Section~\ref{sec:rlvr_setup}), and an energy-based model for the distribution of deterministic policies (Section~\ref{sec:codes}).

\subsection{RLVR Optimized Tabular Policies}
\label{sec:rlvr_setup}
In this class of models, the policy at each timestep relies on a localized context variable
\begin{equation}
 c_t (x,y_{<t}) = (x_{t-n_p}\ldots x_{t+n_f}, y_{t-n_y}\ldots y_{t-1}, t)
\end{equation}
which encapsulates a sliding window of the input, a bounded window of previously emitted tokens, and the current time step (ie token index) $t$. For the boundary cases, we define for every offset $k\in\{-n_p,\dots,n_f\}$,
$x_{t+k}=\text{SEP}$ whenever  $t+k<1$ or $t+k>T$,
and for every $k\in\{1,\dots,n_y\}$,
$y_{t-k}=\text{SEP}$ whenever $t-k<1$. For the boundary cases, we define $x_{t'} = \text{SEP}$ when $t'<1$ or $t'>T$, and similarly $y_{t'}=\text{SEP}$ for $t'<1$. In the following, we take $n_y=1$, and denote the window of each model by $(n_p,n_f)$. We refer to each unique configuration of this context as a cell. For each cell $c$, the policy assigns a discrete probability distribution $\pi(y_t \mid c)$ over the $q$ possible vocabulary tokens. Consequently, the space of all possible policies is simply a product of probability simplices, one per cell. We refer to a policy as tied if the policy is not an explicit function of the token index, so that for all $t,t'$, if $c_t,c_{t'}$ differ only in the token index, then $\pi_{c_t}=\pi_{c_{t'}}$. Note that this implies that untied policies visit each at most once. 
This policy is optimized with RLVR, where the model receives reward $=1$ if $y_T$ equals $s_T$ and $0$ otherwise. Crucially, intermediate tokens are not supervised.  We denote by
$J(\pi)$ the expected reward over the policy's outputs and the input data. This policy is trained by minimizing the following objective
\begin{equation}
    \mathcal{L}(\pi)= - J(\pi)- \beta^{-1} \sum_c H (\pi_c),
\end{equation}
where the second term is an entropy regularizer and we refer to the coefficient $\tau = \beta^{-1}$ as the temperature.  The loss itself is optimized via coordinate-ascent variational inference (CAVI) \citep{blei2017variational}, where we optimize each cell $c$ while fixing the other cells. In this model, we can explicitly define for untied policies the optimal policy for each cell conditioned on all the other ones,  which is given by
\begin{equation}
\label{eq:cavi}
\pi_c(a) \propto \exp\big(d(c)\,Q(c,a)/\tau\big), 
\end{equation}
where $d(c)$ is the probability that rollouts consult cell $c$, and $Q(c,a)$ is the expected reward, with respect to the policies of all other cells, conditioned on action $a$ being taken at cell $c$. See App. \ref{app:dQ} for the CAVI update rules on tied policies. A forward recursion over all input combinations gives visitation probabilities $\dc$, a backward recursion gives the action values $\Qc$, analogous to the algorithm described in \citep{rabiner1989}. Following the policy-gradient theorem \citep{sutton2000}, we then obtain $\J$ and its exact gradient in closed form, with essentially zero estimation error. For full details, see App. \ref{app:dQ}.

\subsection{Statistical Spin-Glass Models of Deterministic Policies}
\label{sec:codes}
We now turn to an a priori different kind of inference model. We study an energy model over policies, plausibly the simplest one that captures the spirit of what RLVR training favors- determinism and high reward. Thus, this distribution is supported on deterministic policies, each of which assigns a single action to each cell. We call such a function a \emph{code}, and write $\sigma:\mathcal{C}\to\mathcal{A}$, with $\sigma_c$ denoting the token assigned to cell $c$ (see Figure~\ref{fig:joint_dictionary}II for a visualization). To favor high reward, we consider a Gibbs distribution with energy (cost) given by $E(\sigma) = -J(\sigma)$, assigning each code probability proportional to $e^{-\beta E(\sigma)}$ at inverse temperature $\beta$,
\begin{equation}
Z(\beta) = \sum_{\sigma \in \mathcal{A}^{\mathcal{C}}} e^{-\beta E(\sigma)}, \qquad p_{\text{spin}}(\sigma) = e^{-\beta E(\sigma)}/Z(\beta).
\label{eq:spinmodel}
\end{equation}
In physics terms, this is a system of $N=|\mathcal{C}|$ $q$-state (Potts) spins with Boltzmann distribution $p(\sigma) \propto e^{-\beta E(\sigma)}$. Here, $J(\sigma)$ is the same expected reward as defined in \ref{sec:rlvr_setup}, evaluated at the deterministic policy $\sigma$, which is explicitly given by
\begin{equation}
J(\sigma) \;=\; \mathbb{E}_{x}\big[ \sum_{y\in\mathcal{A}^{T}}
\delta_{y_T,\,s_T(x)} \prod_{t=1}^{T}
\delta_{\sigma_{\,c_t(x,\,y_{<t})},\;y_t} \big],
\label{eq:trajexp}
\end{equation}
Note that this is a non-local, distributed interaction reminiscent of, but distinct from, a canonical spin-glass system. Both share a discrete phase space with non-local interactions; the difference is that our model's interactions are complex (quasirandom) and contain a gauge symmetry rather than genuinely random, as in the spin-glass setting. See Fig. (\ref{fig:dictionary}) for a summary of the RL/physics terminology correspondence, and App. \ref{app:gauge} for a more detailed comparison to classic spin models. 

This model is part of a well-studied class of physical systems, bringing with it a rich theoretical toolbox for analyzing such distributions \citep{mezardmontanari}. In some cases, the model can be solved analytically (see Sec. \ref{sec:parity}), and more generally, since codes take values in a finite space which can in principle be enumerated over to compute various aspects of the system. Alongside these theoretical tools, Monte Carlo (MC) methods allow us to sample the distribution directly \citep{newman1999monte}.

Notwithstanding, we note that this particular choice of energy function is itself just one choice among many. Moreover, the window and tying restrictions of Section~\ref{sec:rlvr_setup} are special cases of a more general family. Writing $E_{\text{win}}$ and $E_{\text{tie}}$ for penalties that respectively discourage codes from treating histories differently within a window, or across tied generation steps, we consider
\begin{equation}
E_{\mathrm{tot}}(\sigma) \;=\; -J(\sigma)
\;+\; g_{\mathrm{win}}\,E_{\mathrm{win}}(\sigma)
\;+\; g_{\mathrm{tie}}\,E_{\mathrm{tie}}(\sigma),
\label{eq:family}
\end{equation}
recovering the complete parameterisation at $g_{\text{win}} = g_{\text{tie}} = 0$, and the untied and tied windowed classes as $g_{\text{win}}, g_{\text{tie}} \to \infty$. We note in passing that a transformer's implicit bias or pretraining could similarly be represented as a penalty term of this kind.

\begin{figure}[t]
    \centering
    \renewcommand{\thesubfigure}{\Roman{subfigure}}
    
    \begin{subfigure}[c]{0.52\textwidth}
        \centering\small
        \begin{tabular}{p{3.8cm} l} 
        \toprule
        Reinforcement Learning &  Physics term \\
        \midrule
        deterministic policy (code)  & spin configuration $\code$ \\
        expected reward $\J$ &  minus the energy \\
        one table entry (cell)  & one $q$-state (Potts) spin \\
        stochastic policy & variational product state \\
        entropy coefficient $\temp = \beta^{-1}$  & temperature \\
        $\dc\,\Qc$  & local field at a site \\
        relabelling of CoT tokens  & gauge transformation \\
        set of near-trap configs  & pure state \\
        \bottomrule
        \end{tabular}
        \caption{Mapping dictionary}
        \label{fig:dictionary}
    \end{subfigure}\hfill
    \begin{subfigure}[c]{0.45\textwidth}
        \centering
        \includegraphics[width=\linewidth]{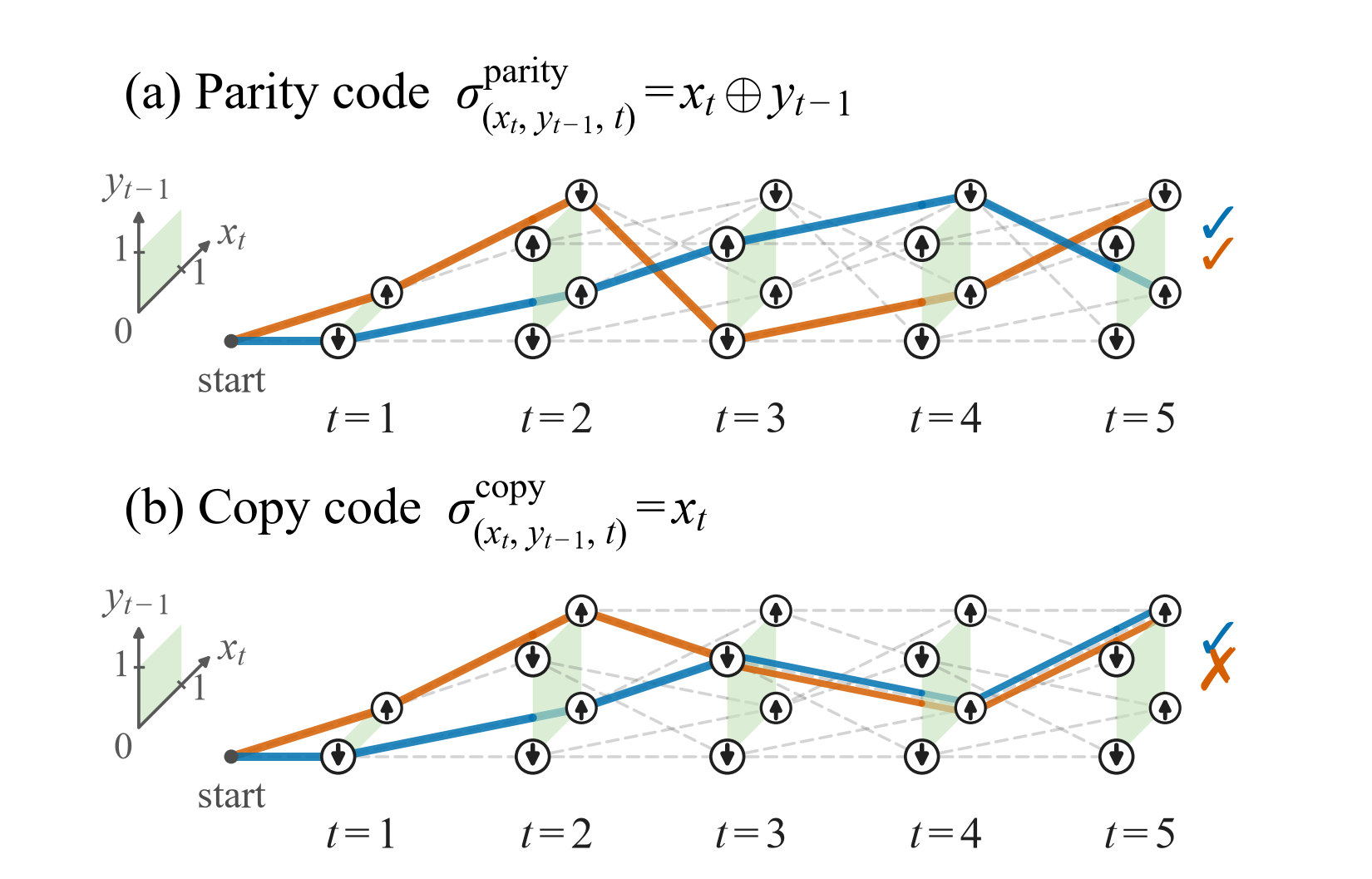} 
        \caption{Visualized mapping}
        \label{fig:sub_image}
    \end{subfigure}
    
    \vspace{-0.5em} 
    
    \caption{The RLVR and spin-glass mapping. (\textbf{I}) Relating concepts from reinforcement learning to spin-glass models. (\textbf{II}) Two codes on the $(0,0)$ window (i.e.\ $\cell_t = (x_t,y_{t-1})$), for $q=2$ and $T=5$. Each circle is a cell $(x_t, y_{t-1}, t)$; the arrow inside it is the code's output at that cell, up for token 1 and down for token 0. Dashed gray edges trace the code's action on all input sequences. Two inputs differing only in $x_1$ are highlighted: {\color[HTML]{0072B2}$x=[0,1,0,1,1]$} and {\color[HTML]{D55E00}$x=[1,1,0,1,1]$}.Both trajectories of the parity code (a) converge to the correct answer, whereas the copy code (b) succeeds on only one of the two, the same success rate as chance.}
    \label{fig:joint_dictionary}
\end{figure}

\section{Mapping the spin-glass model to RLVR}
\label{sec:dictionary}

Up to this point, we have made no explicit connection between the RLVR objective and our energy-based spin model. In this section we make this comparison more rigorous. We show that the local maxima of the reward as a function of policies are continuously connected to codes with the same reward; we then use the spin model to obtain a bound on the RLVR loss, which is important for regulator effects at finite temperatures; and we show that the Monte Carlo sampler's traps in the spin model coincide with local minima of RLVR training. 

\subsection{Spin-Glass Codes as Reward Maximizers}
\label{sec:codes_reward}
Our first result is given by the following statement
\begin{statement}
\label{thm:neutral_paths}
We consider the space of policies as a product of simplices embedded in $\mathbb{R}^{q\times |\mathcal{C}|}$. Let $\pi^*$ be an untied policy such that there exists an open set around it where $\pi^{*}$ obtains a maximal reward, then it can be connected by a continuous path of fixed reward to a deterministic policy $\sigma^*$ such that $J(\sigma^*) = J(\pi^*)$.
\end{statement}
We show this for a simplified case here; the full proof follows a similar structure (App.~\ref{app:neutral_path}). Let $\pi^*$ locally maximize $J$, and let $c$ be a cell that the code consults with at which $\pi^*_c$ is non-deterministic. For simplicity, suppose only two actions $a,b$ are possible at this cell, with $\pi_c^*(a) = s$, $\pi_c^*(b) = 1-s$. By multilinearity of $J$, we have, up to constants, $J \propto s\,Q(c,a) + (1-s)\,Q(c,b)$. Since $\pi^*$ is a local maximum of the reward, we require $Q(c,a) = Q(c,b)$, so $J$ is constant in $s$, and we are free to assign any probability we wish to $a$ and $b$. The edge cases $s \in \{0,1\}$ are codes. 
The proof in App.~\ref{app:neutral_path} extends this approach to policies with larger support and more cells.

\subsection{Free Energy as an RLVR Loss Bound}
\label{sec:free_energy_bound}
A tabular RLVR policy is, by construction, factorized across cells: it assigns an independent distribution $\pi_c$ to each cell $c$. Viewing a cell as a spin site, this is exactly the structure of a mean-field ansatz in physics, a distribution that assumes no correlations between sites. This ansatz is typically correct when each spin interacts with many other spins, as is the case for our model. Motivated by this observation, we consider the fully factorized distribution over spins $\hat{\pi} = \bigotimes_c \hat{\pi}_c$ (so that for codes drawn from this distribution, we have $\hat{\pi}_c(a) = \mathbb{E}_{\sigma \sim \hat \pi }[\delta_{\sigma_c, a}]$). This variational distribution can be used  to bound the free energy of the true distribution, $F_{\text{spin}}=-\beta^{-1}\log Z$ with the Gibbs-Bogoliubov inequality \citep{mezardmontanari}:
\begin{equation}
F_{\text{var}}[\hat{\pi}] \;=\; -\tilde{J}(\hat{\pi}) - \beta^{-1}\sum_c H(\hat{\pi}_c) \;\ge\; F_{\text{spin}} \qquad
\tilde J(\hat \pi) = \mathbb{E}_{\sigma\sim \hat \pi}\big[J(\sigma)\big].
\label{eq:variational}
\end{equation}
We note that upon substituting $\tilde{J}$ with $J$, we obtain that $F_{\text{var}}[\pi]= \mathcal{L}(\pi)$ where $\pi$ is a tabular policy. However, $\tilde{J}$ is not, in general, the same as the RLVR reward $J$. In $\tilde{J}$, every cell's action is sampled from $\hat{\pi}$ independently, whereas in $J$, actions are generated sequentially via a rollout. The two rewards coincide whenever no non-deterministic cell is visited more than once during a rollout. This holds exactly for untied tabular models, and approximately for tied tabular models and near zero temperature, where such revisits are rare and cells are largely deterministic, respectively. See Appendix~\ref{app:correspondence} for the full derivation. 
When these rewards coincide, the RLVR loss of any tabular policy $\pi$ is equal to the variational free energy, so that
\begin{equation}
\mathcal{L}(\pi) = F_{\text{var}}[\pi] \;\geq\; F_{\text{spin}}.
\label{eq:identity}
\end{equation}

Thus, the free energy of the true spin model distribution sets a lower bound, which is nontrivial at finite $\beta$,  on what any tabular policy trained with RLVR can achieve (i.e. policies that solve Eq. \ref{eq:cavi}). 

The tightness of this bound depends on how well $\hat \pi$ approximates $p_{\text{spin}}$, namely how good the mean-field approximation is for the configuration encountered in Monte Carlo sampling. For some models, as is demonstrated in Section \ref{sec:parity}, $F$ can be solved explicitly, giving a concrete, quantitative bound to the RLVR loss. More generally, the discrete nature of the system can often be advantageous, making the configuration space more likely to be enumerable, and so more tractable to solve outright, offering a route to bounding the loss. We comment that multiplying the bound in Eq.~\eqref{eq:identity} by $\beta$ reveals a machine-learning analogue: maximizing $\beta \tilde{J}(\hat{\pi}) + \sum_c H(\hat{\pi}_c)$ is precisely maximizing the evidence lower bound (ELBO), $\beta \tilde{J}(\hat{\pi}) + \sum_c H(\hat{\pi}_c) \leq \log Z$.  

\subsection{Spin-Glass Pure States as Training Traps}
\label{sec:pure_states_traps}
In spin-glass theory, it is well established that Monte Carlo sampling tends to become trapped in \emph{pure states}: configurations in which $\sigma_c$ and $\sigma_{c'}$ are approximately uncorrelated for $c \neq c'$ \citep{mezardmontanari}. The full distribution over spins can then be understood as a mixture over such pure states. Since these pure states are, by definition, fully factorized across cells, they lie within the same variational family $\hat{\pi} = \bigotimes_c \hat{\pi}_c$ considered above. Among all such factorized distributions, the most likely pure states are precisely those that locally minimize the variational free energy $F_{\text{var}}$. Thus, we expect that when Monte Carlo converges to a steady state, the corresponding distribution would have weakly correlated spins, and its marginals per cell would be approximate solutions to Eq. \ref{eq:cavi}. In Fig. ~\ref{fig:trapsfixed} we demonstrate empirically that this is indeed the case for the trapped models that we studied.

\begin{figure}[t]
\centering
\includegraphics[width=\linewidth]{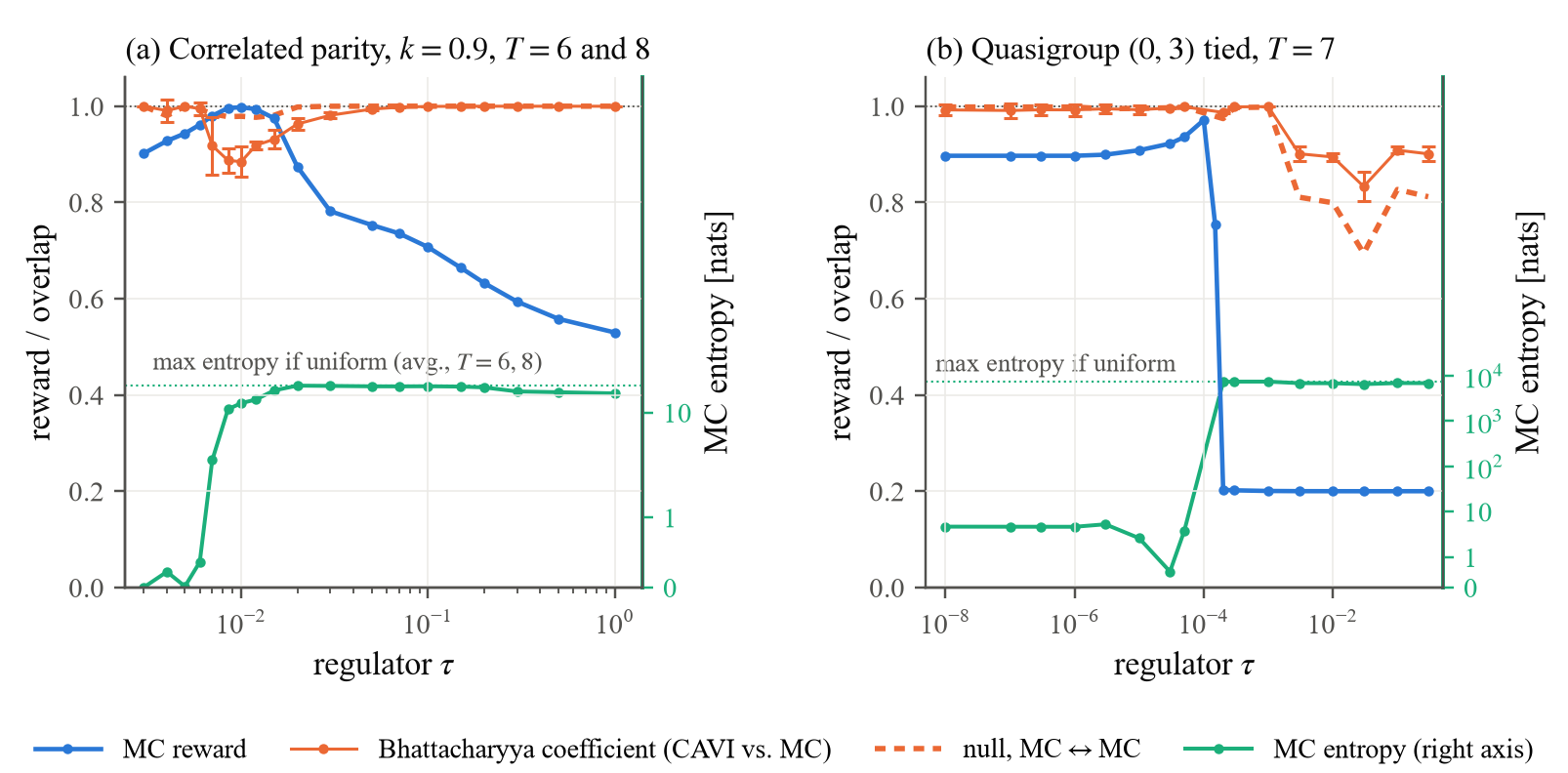}
\caption{MC stationary distributions against the CAVI endpoints seeded at them, as a function of the regulator $\tau$: (a) the 27 correlated-input parity traps, (b) the twelve tied quasigroup traps. Left axis: MC reward, averaged over all traps and replicas; Bhattacharyya coefficient between the CAVI endpoint and the MC histogram, mean $\pm$ standard deviation over the finite-temperature traps; and its null. Right axis: the entropy $\sum_c H(\pi_c)$ of $\pi^{\mathrm{MC}}$, averaged over the same traps; dotted lines mark its value when every cell is uniform. Open markers: fewer than half of the traps enter the average at that $\tau$; missing points: none does.}
\label{fig:trapsfixed}
\end{figure}

\section{Some exact results on untied models}
\label{sec:exact_results_spin}
\subsection{Untied classes hold no defended minima}
\label{sec:escape}

A key empirical result of this paper is that, for a wide class of untied models, namely untied $(0,0),(2,1),(1,1),(3,0)$ window models with uncorrelated inputs, the landscape has no local minima, as evidenced by Monte Carlo optimization at vanishing temperature consistently sampling codes with $J=1$. See Tab. \ref{tab:shortfall} for full results. Here we support this with the following analytical result

\begin{theorem}[local search reaches optimal reward]
\label{thm:escape}
Consider the untied $(0,0)$ window (i.e.\ $\cell_t = (x_t,y_{t-1})$) class on a quasigroup fold with inputs drawn
uniformly, and local dynamics that change one cell at a time and
accept any change that does not lower $\J$. From every code there is a
finite chain of accepted changes ending at a code with $\J=1$. 

\end{theorem}

The proof (Appendix~\ref{app:escape}) proceeds in two parts. First, we show that a \emph{collision-free} code, one in which no two histories in different task states are ever assigned the same output token, reaches $J=1$ in a finite number of steps that do not decrease the reward. 
Second, we show how to remove collisions iteratively, without ever decreasing the reward. At a colliding step, every cell is first set to the token that is individually optimal for the true state it leads to. Since there are exactly $q$ possible true states and $q$ possible tokens at this step, any two states left sharing a token after this adjustment leave exactly one token unused. We show how the code can be set to route through the unused token without decreasing the reward. Applying this at each successive colliding step, and then treating the resulting collision-free code by the first part, proves the theorem. 

\textbf{Corollary 4.2} \emph{A non-decreasing reward escape path exists also in policy space when starting from a code.} Proof:  Note that due to multilinearity of $J$, a linear extrapolation between $\pi_c$ and $\pi_c'$, with rewards $J$ and $J'$ respectively, results in a linear interpolation of the reward. Since the theorem relies only on such local policy changes, it implies a piecewise linear extrapolation in policy and reward space. 

\textbf{Corollary 4.3} \emph{A non-decreasing reward escape path exists from any policy.} Proof: If the policy is in a local maximum, it is connected to a code via statement \ref{thm:neutral_paths}. If it is not in a maximum, then, by definition, a path with improving rewards exists.

\subsection{ Exact solution for parity }

\label{sec:parity}
Section \ref{sec:free_energy_bound} showed that the spin model's free energy bounds the RLVR loss, without evaluating that free energy explicitly. This section does so for the simplest member of the task family, the two-element group, $q=2$, whose
fold is the parity of the input bits. This section computes an explicit expression for the partition function for the untied $(0,0)$ window, and predicts the existence of a first-order phase transition with $\beta$.

Following the terminology of reinforcement learning under partial
observability \citep{whitehead1991,mccallum1996}, we say that a code
\emph{aliases} at step $t$ when two histories whose task states $s_t$ differ
produce the same emitted token $y_t$. What makes this the relevant notion is
what the rest of the computation can see: later steps read $y_t$ and the
input symbols still inside their window, but not the symbols that have left
it, so any distinction absent from $y_t$ is absent for good.

Collect the
code's decisions at step $t$ for input symbol $x$ into a $q\times q$ matrix
$[\Sigma_{t,x}]_{y_t,\,y_{t-1}} \;=\; \delta_{\sigma_{(t,x,y_{t-1})},\,y_t}$
which we refer to as a \emph{link matrix}. It has exactly one entry equal to one in each column, so that it is a map
$\alphabet\to\alphabet$. For the (0,0) window, there are two such matrices per $t$: $\Sigma_{t,0}$ and $\Sigma_{t,1}$. Call a step $t$ \emph{dead} if both of its link matrices ($\Sigma_{t,x_t}$) are constant as a function of $x_t$, or both are rank-1. Both are forms of aliasing, as they make us lose track of the state ($s_t=B[y_{t-1},x_t]$) via loss of $x_t$ and $y_{t-1}$ information, respectively. Note that the presence of any dead
step forces a pure chance reward ($\J = 1/2$) regardless of subsequent
steps. We say that the code is \emph{alive} if it has no dead steps. Next,
define a \emph{half-reset} as a step where only one of the two matrices has
rank 1.

Denote by $k$ the number of half-resets and consider a live code with
$k=0$. By being non-dead and non-half-reset, each step must follow the
correct processing of $x_t,y_{t-1}$ or the anti-correct one. This provides
$2^{T-1}$ codes with $\J = 0$ and $2^{T-1}$ codes with $\J = 1$. Next,
consider a single half-reset ($k=1$) at $t = t_1$. This means $t_1$ ignores
its CoT token for one of the $x_{t_1}$ values ($x_{d_1}$). As a result, the
reward on inputs for which $x_{t_1} = x_{d_1}$ is $\J = 1/2$ regardless of
subsequent steps. However, conditioning on $x_{t_1}\neq x_{d_1}$ the half
reset is effectively not present, leading, by similar logic to the $k=0$
case, to $2^{T-2}$ codes with $\J = 1$ and $2^{T-2}$ codes with $\J = 0$ on
that branch. Weighing the two equally probable events,
$x_{t_1} = x_{d_1}$ and $x_{t_1}\neq x_{d_1}$, we find codes with
input-averaged reward $\J = 1/4 + 1/2$ and codes with $\J = 1/4$.

Generalizing this logic to arbitrary $k$, taking into account combinatorial
factors and dead steps, one obtains the model's normalizing constant --- its
partition function --- in closed form (see App. \ref{app:parityZ}),
\begin{equation}
Z(\invtemp,T) \;=\; e^{\invtemp/2}\left[\,2^{4T-2} - 2\cdot 10^{\,T-1}
\;+\; 2^{T}\sum_{k=0}^{T-1}\binom{T-1}{k}\,4^{k}
\cosh\!\left(\frac{\invtemp}{2^{\,k+1}}\right)\right].
\label{eq:parityZ}
\end{equation}
The average reward is then given by $\mathbb{E}_{\code \sim p_{spin}}[J(\sigma)]=\partial_{\beta} \log(Z(\beta,T))$, in large $T$ limit, it shows a first order phase transition at $\tau_c = 1/(6T\log(2))$, where $J$ goes from $1$ to chance ($1/2$) where the entropy of all the dead codes wins over the reward scale (see App. \ref{app:parityZ:transition}). 

\section{Further Experimental Results}
\label{sec:further_results}

The results of previous sections raise several natural questions that are addressable numerically. The first concerns the tightness of our bound and the importance of diffusive-barrier or barren plateaux in the landscape. In Table \ref{tab:shortfall} we measure that difference at negligible temperatures, showing that for many untied tasks RLVR falls much below the spin model which consistently obtaines near perfect reward. The success of the Monte-Carlo sampling suggests that diffusive barriers for the spin model are mild. 

This apparent gap motivates us to look for potential improvements on the RLVR side  (see App. \ref{app:cures} for further detail). The success of Monte Carlo samplers motivates the use of periodic resets as a direct means of expediting diffusion.
The existence of a phase transition in the parity task motivates us to consider large regulator coefficients (temperatures) to reduce free-energy and entropic barriers. Figure \ref{fig:temperature_untied}, left panel, shows the results of these, with about half the models responding well to high entropy regulator levels. Resets (solid markers) work even better for this class. 


\begin{figure}[t]
\centering
\includegraphics[width=\textwidth]{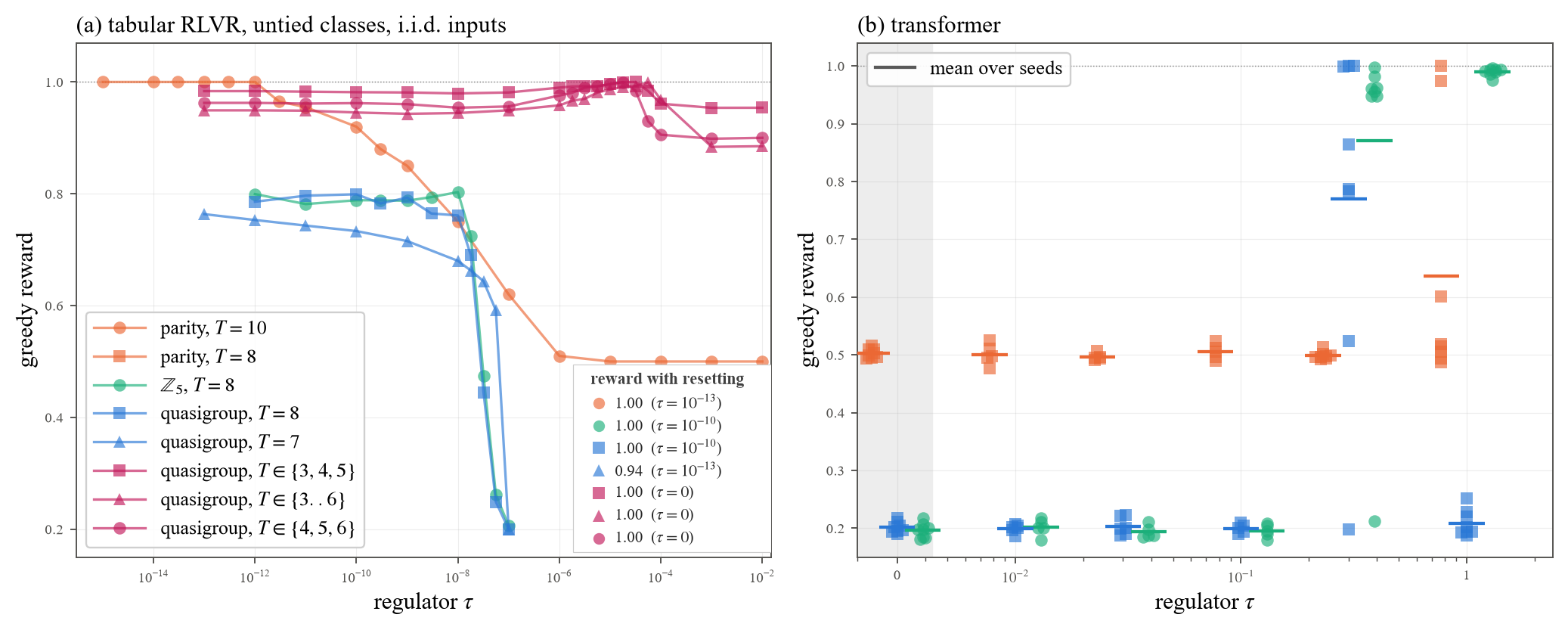}
\caption{\textbf{The entropy regulator on the untied classes.}
\textbf{(a)} Tabular RLVR from a uniformly random policy at a constant regulator $\tau$, scored by the greedy reward of the
endpoint policy. Each point is a mean over independent runs: $100$ for parity (window $(0,0)$), $240$ for the
two classes at $T=8$ using window $(0,0)$ ($\mathbb{Z}_5$, quasigroup), $12$ for the narrow-window classes using
window $(2,1)$ (at $T\in\{3,4,5\}$ and $T\in\{4,5,6\}$) and window $(1,1)$ (at $T\in\{3,\dots,6\}$), and for
quasigroup using window $(0,3)$ at $T=7$. The inset table summarises reset interventions of the same class ($2\%$ of the cells redrawn
uniformly every $10$ iterations for $100$ cycles): the reward attained and the $\tau$ of the run it summarises.
\textbf{(b)} The one-block transformer under the gated curriculum, REINFORCE with an entropy bonus
$\beta$ (occupancy convention), read as greedy accuracy on held-out inputs of length $8$ at the
final step. One point is shown per seed, bar is the mean over seeds. The transformer has no direct counterpart of
tabular cell resets.}
\label{fig:temperature_untied}
\end{figure}

From a different angle, the lack of a reward barrier in the landscape motivates us to look for variations of the task that induce such barriers. We consider tied tables or correlated inputs, wherein the inputs are drawn from a Markov chain with a $k>1/q$ chance of emitting the previous token. Both variations violate different lemmas of our escape proof, and indeed in Tab. \ref{tab:trapsexist}, we numerically establish the existence of true energy barriers. We note in passing that, for tied models, these barriers are microscopic yet common. Namely, they have a reward barrier equal to failing on a few specific input combinations, but occur for most initialization seeds. For correlated inputs, barriers are macroscopic but rare (i.e. involve a finite fraction of input combinations but only occur for about $12\%$ of initialization seeds).

\subsection{Transformers on the same task}
\label{sec:transformer}

We consider tabular policies as a toy model for transformers, which raises a natural question: do the escape paths we establish for tabular policies also exist for transformers trained with RLVR? We would expect an affirmative answer if two things hold: a transformer can express a tabular policy, and it can traverse the escape route between one code and the next without losing reward. We show below that both do hold. Motivated by these results, we train transformers on the same tasks under sufficient entropy regularization, and show that they often converge to 
$J=1$ consistent with a benign landscape (Fig. \ref{fig:temperature_untied} panel b).

Starting with expressivity, we verify that a modest $791$k-parameter transformer can be
brought to express random untied tabular policies, 
with wider windows taking longer to train, suggesting an implicit bias
favoring narrow tables (see App.~\ref{app:containment},
Table~\ref{tab:containment}).

We further test how well such a transformer follows an escape path of Monte-Carlo. To this end, we locate distinct codes in
the chain, train a transformer on one such code
($\sigma_i$), save its weights ($\theta_i$), then fine-tune it on the next distinct
code ($\sigma_{i+1}\neq\sigma_i$, on visited cells) and obtain the resulting
weights ($\theta_{i+1}$). Notably, such fine-tuning never fails to reach perfect
agreement: At $T=6$, all $189$ points reproduce their code exactly on all $5^6=15{,}625$
inputs. We then consider a linear interpolation between $\theta_i$ and
$\theta_{i+1}$ and test whether this incurs a drop in reward along the path.
Interestingly, for the large majority of sequential codes --- $159$ of $180$
steps --- no such drop is observed, and for the remaining small fraction the drop in the reward is at most $2\times10^{-3}$. See Tab. \ref{tab:escape} and App.~\ref{app:containment} for further details.

The above results suggest transformers should inherit the simplifying landscape
features of untied tabular models. Motivated by this, we study the
effects of strong entropy regulators (standard occupancy-weighted regulator) on
the same three tasks as in the tabular setting, together with curriculum learning to prevent signal-to-noise
issues. As discussed in App.~\ref{app:cures} and shown in
Fig.~\ref{fig:temperature_untied}, without a
regulator, and at every coefficient up to $\beta^{-1}=0.1$, no seed of any task holds
more than length $T=2$ and every seed is at chance at length $T=8$. At $\beta^{-1}=0.3$
five of eight seeds on $\mathbb{Z}_5$ and three of eight on the quasigroup pass
$0.95$ at length $T=8$, and they carry the curriculum far beyond it, to $T=18$ on $\mathbb{Z}_5$ and to length $T=20$ on three of the quasigroup seeds. 

Notably, regulator choices here far exceed standard practice
\citep{ahmed2019, cui2025entropy}, by one to two orders of magnitude. Around $\beta^{-1}=0.3$, the policy carries $42\,\%$ of the maximum entropy ($\ln 5$). We note in passing that before this strong regulator choice, we made many changes to optimization (REINFORCE,GRPO,Expert-Iteration), curricula, adaptive batching, eps-exploration, and combinations thereof, all failed to learn at $T=4$.

\section{Discussion}
\label{sec:discussion}

In this work, we established a mapping between RLVR on tabular models and a spin-glass model whose spins take values in action space and sit at cells/table-entries. Configurations of spins amount to deterministic RL policies. The spin-models appear to enjoy a larger degree of tractability compared to the original RLVR problem, and their free energy lower-bounds the RLVR loss. Gaps between plain RLVR and these lower-bounds motivated us to consider strong regulator effects. Surprisingly, they facilitate RLVR training from scratch for both tabular models and transformers. We further describe several scenarios in which both the spin-model and RLVR end up trapped in valleys of the loss. 

We hope to leverage this framework to further study fundamental limitations of RLVR. In particular, it would be interesting to apply it to harder algorithmic tasks which may require invoking the full spin-glass framework, replica-symmetry breaking included. In addition, it would be interesting to incorporate the effects of pre-training and implicit biases of more complex architectures via the inclusion of prior terms.  

{\bf Limitations.}  Pure tabular models are a rough caricature of transformers and thus offer only a qualitative link. We further focused on untrained models, based on the concept of skill compositions. Extending this to real pre-trained models carries in much complexity. While prior terms, including $KL$-penalties, can close some of these gaps qualitatively, their specific relation to RLVR training of networks remains to be studied.  The strong regulator choices which facilitated learning here are likely damaging for actual pre-trained models.  While some elements of spin-glass theory have been used here, in particular the notion of pure-states and their link to local Monte-Carlo updates, it is yet to be determined whether important concepts such as replica symmetry breaking carry through to these specialized spin models.

\bibliographystyle{iclr2027_conference}
\bibliography{refs}

\appendix
\newpage
\section{Definitions of various RL quantities}
\label{app:dQ}

Let $x=x_1...x_T$ be the input sequence drawn
with probability $P(x)$ (uncorrelated or Markovian); at step $t$ the policy reads the cell
$\cell_t(x,y_{<t})$ determined by the input and the tokens emitted so far, it then emits
$y_t\sim\pol_{\cell_t}$, and the rollout ends with reward $R(x,y)$ (where $y$ is the full CoT).  For a CoT prefix
$y_{\le t}$ let
\begin{equation}
V_t(x,y_{<t}) = \mathbb{E}_{\pol}\big[R(x,y)\,\big|\,x,\,y_{<t}\big]
\label{eq:rtg}
\end{equation}
be the expected reward conditioned on CoT prefix $y_{<T}$ and $x$, where $y_t..y_T$ follow $\pol$.
The \emph{visitation} ($d(c)$) of a cell and its \emph{action values} ($d(c)Q(c,a)$) are
\begin{align}
\dc &= \sum_x P(x)\sum_t\sum_{y_{<t}} P_\pol(y_{<t}\mid x)\,
       \mathbf{1}\big[\cell_t(x,y_{<t})=\cell\big], \label{eq:d}\\
\dc\,\Qc &= \sum_x P(x)\sum_t\sum_{y_{<t}} P_\pol(y_{<t}\mid x)\,
           \mathbf{1}\big[\cell_t(x,y_{<t})=\cell\big]\;V_{t+1}\big(x,(y_{<t},a)\big). \label{eq:dQ}
\end{align}
So $\dc$ is the expected number of times a rollout consults $\cell$, and $\Qc$ is
the expected reward of a rollout that takes action $a$ at cell $c$ and
follows $\pol$ everywhere else. Notably for untied policies, the choice of $c$ fixes a single $t$. 

These quantities arise when differentiating the reward with respect to an action $a$ at cell $c$, specifically 
$\partial\J/\partial\pol_\cell(a)=\dc\,\Qc$ for both tied and untied models. For untied models, these gradients are independent of $\pol_\cell(a)$, and the balancing of the gradient with the entropy terms yields the CAVI equations of the main text. For tied models, this independence does not hold, and gradient descent on $\pol_\cell(a)$ amounts to CAVI with a strong damping factor. In practice we worked with a 30\% damping factor; since a fixed point of the damped update is a fixed point at any damping, and every tied endpoint we report has residual exactly zero, the stationary points quoted here do not depend on that choice.

The reward depends on
the input only through the running fold $s_t=s_{t-1}\circ x_t$, and the policy
reads the input only through the window $W_t=(x_{t-n_p},\ldots,x_{t+n_f})$, with
positions outside $1,\ldots,T$ reading the boundary symbol. Because the inputs are
independent or Markov, a rollout is also a Markov chain on the joint CoT and input-window state
\begin{equation}
z_t=(s_{t-1},\,W_t,\,y_{t-1}),
\end{equation}
Step
$t$ emits $a\sim\pol_{\cell_t(W_t,y_{t-1})}$, updates the $s_{t-1}$ with the entry $x_t$
of the window, shifts the window, and draws the one new input $v=x_{t+n_f+1}$ from
the Markovian input distribution ($P(v \mid W)$), allowing us to define a transfer matrix 
\begin{equation}
M_t(z_{t+1}\mid z_t)=\pol_{\cell_t(W_t,\,y_{t-1})}(y_t)\;P\big(x_{t+n_f+1}\mid W_t\big),
\qquad
z_{t+1}=\big(s_{t-1}\circ x_t,\;W_{t+1},\;y_t\big),
\end{equation}
with $s\circ x_t=x_t$ at
$t=1$. Its forward and backward passes are
\begin{align}
\alpha_{t+1}(z') &= \sum_z\alpha_t(z)\,M_t(z'\mid z), &
\beta_t(z) &= \sum_{z'}M_t(z'\mid z)\,\beta_{t+1}(z'),
\end{align}
started from $\alpha_1$, the distribution of the first window with $y_0=\mathrm{BOS}$,
and from $\beta_{T+1}(z)=\mathbf{1}[y_T=s_T]$. They give
\begin{align}
\J &= \sum_z\alpha_t(z)\,\beta_t(z)\quad\text{for every }t,\\
\dc &= \sum_t\sum_{z:\,\cell_t(z)=\cell}\alpha_t(z),\\
\dc\,\Qc &= \sum_t\sum_{z:\,\cell_t(z)=\cell}\alpha_t(z)\;
  \mathbb{E}_v\big[\beta_{t+1}(z')\,\big|\,z,\,y_t=a\big],
\end{align}
where the last expectation is over the new input alone. The state has
$q^{\,n_p+n_f+3}$ values whatever the length, so the cost is proportional to
$T\,q^{\,n_p+n_f+5}$, linear in $T$. Under tying the cell does not depend on $t$ and
the sums over $t$ collect every consultation, as in Eq.~\eqref{eq:d}. For classes
trained on several lengths, $\J$, $\dc$ and $\dc\,\Qc$ are averaged over the lengths
with equal weight before $\Qc$ is formed.

\section{Local maxima of the tabular RLVR objective are connected by neutral paths to deterministic codes}
\label{app:neutral_path}

\subsection{Setting}

\begin{definition}[Task data]\label{def:task}
An \emph{action alphabet} $\A$ with $|\A| \ge 2$; an \emph{input set} $\X$ with a probability distribution $P$; To accommodate multi-$T$ datasets we defined the \emph{length} $T(x) \in \{1,2,\dots\}$ for each $x \in \X$, with $T_{\max} := \max_x T(x)$; and a \emph{reward} $R(x,y) \in \R$ for every $x \in \X$ and every \emph{trajectory} $y = (y_1,\dots,y_{T(x)}) \in \A^{T(x)}$. For $1 \le t \le T(x)$ write $y_{<t} := (y_1,\dots,y_{t-1}) \in \A^{t-1}$ (the empty word when $t=1$). The reward is arbitrary.
\end{definition}

\begin{definition}[Cells and routing]\label{def:routing}
A finite set $\C$ of \emph{cells} with a \emph{step map} $\step : \C \to \{1,\dots,T_{\max}\}$, and for every $x \in \X$ and $t \in \{1,\dots,T(x)\}$ a \emph{routing map}
\[
\kappa_t(x,\cdot) : \A^{t-1} \to \C \qquad\textup{with}\qquad \step\big(\kappa_t(x,y_{<t})\big) = t \ \textup{ for all } y_{<t}.
\]
$\kappa_t(x,y_{<t})$ is the cell (table row) consulted at step $t$ on input $x$ after the tokens $y_{<t}$ have been emitted. It does not depend on any policy. The constraint $\step(\kappa_t(\cdot)) = t$ is the only property of the routing that is used.
\end{definition}

\begin{definition}[Policies, objective, codes]\label{def:J}
Let $\simp(\A) := \{p \in \R^{\A} : p(a) \ge 0,\ \sum_a p(a) = 1\}$ and $\Pol := \prod_{c\in\C}\simp(\A) \subset \R^{\C\times\A}$. A \emph{policy} is $\pi = (\pi_c)_{c\in\C} \in \Pol$. For $x\in\X$ and $y \in \A^{T(x)}$,
\begin{equation}\label{eq:traj}
P_\pi(y\mid x) := \prod_{t=1}^{T(x)} \pi_{\kappa_t(x,y_{<t})}(y_t), \qquad
J(\pi) := \sum_{x\in\X} P(x) \sum_{y\in\A^{T(x)}} P_\pi(y\mid x)\,R(x,y).
\end{equation}
A \emph{code} is a map $\sigma : \C \to \A$; it is identified with the deterministic policy $e_\sigma$, $(e_\sigma)_c := e_{\sigma_c}$, where $e_a \in \simp(\A)$ is the point mass at $a$. We write $J(\sigma) := J(e_\sigma)$. The codes are the vertices of $\Pol$.
\end{definition}

\subsection{Multilinearity}

\begin{lemma}[Single visit]\label{lem:single}
For every $x$ and $y \in \A^{T(x)}$, the cells $\kappa_1(x,y_{<1}),\dots,\kappa_{T(x)}(x,y_{<T(x)})$ are pairwise distinct.
\end{lemma}
\begin{proof}
$\step(\kappa_t(x,y_{<t})) = t$, so cells consulted at different steps have different images under $\step$.
\end{proof}

Formula \eqref{eq:traj} defines $J$ as a polynomial in the coordinates $\{\pi_c(a)\}_{c\in\C,a\in\A}$ on all of $(\R^{\A})^{\C}$; call $\{\pi_c(a)\}_{a\in\A}$ the \emph{block} of cell $c$.

\begin{lemma}[Multilinearity]\label{lem:multilinear}
Every monomial of the polynomial $J$ contains at most one variable from each block.
\end{lemma}
\begin{proof}
For fixed $(x,y)$ the product in \eqref{eq:traj} has one factor $\pi_c(y_t)$ from each of $T(x)$ cells, which are pairwise distinct by Lemma~\ref{lem:single}. $J$ is a linear combination of such products.
\end{proof}

\subsection{Local maxima, faces, tangent spaces}

\begin{definition}[Local maximum]\label{def:locmax}
$\pi^\star \in \Pol$ is a \emph{local maximum} of $J$ on $\Pol$ if there is an open set $\Nb \subset \R^{\C\times\A}$ containing $\pi^\star$ with $J(\pi) \le J(\pi^\star)$ for all $\pi \in \Nb\cap\Pol$. (The point $\pi^\star$ may lie on the boundary of $\Pol$; local maximality will be invoked only against competitors that are shown to lie in $\Pol$.)
\end{definition}

\begin{definition}[Support, compatible face, compatible codes]\label{def:face}
For $p \in \simp(\A)$ let $\supp(p) := \{a \in \A : p(a) > 0\}$, a set of \emph{actions}. For $B \subseteq \A$ let $\simp(B) := \{p \in \simp(\A) : \supp(p) \subseteq B\}$; this set is closed and contains the point masses $e_a$, $a \in B$. For $\pi^\star \in \Pol$ define
\[
F(\pi^\star) := \prod_{c\in\C}\simp\big(\supp(\pi^\star_c)\big).
\]
$F(\pi^\star)$ is the smallest face of $\Pol$ containing $\pi^\star$; it is closed, and $\pi^\star$ lies in its relative interior, automatically: a point on the relative boundary of $\simp(\supp(\pi^\star_c))$ would have $\pi^\star_c(a) = 0$ for some $a \in \supp(\pi^\star_c)$, contradicting the definition of the support. We comment that \emph{relative interior} has its usual meaning in convex analysis: the interior relative to the affine hull. If $|\supp(\pi^\star_c)| = 1$, the factor $\simp(\supp(\pi^\star_c))$ is a single point; its affine hull is that point, so its relative interior is the point itself and its relative boundary is empty. 

A code $\sigma$ is \emph{compatible with $\pi^\star$} if $\pi^\star_c(\sigma_c) > 0$ for every $c$, i.e.\ $\sigma_c \in \supp(\pi^\star_c)$. The compatible codes are exactly the vertices $e_\sigma$ of $F(\pi^\star)$.
\end{definition}

\begin{definition}[Tangent spaces]\label{def:tangent}
For $\pi^\star\in\Pol$ and $c \in \C$ let
$V_c := \{v \in \R^{\A} : \sum_a v(a) = 0,\ v(a) = 0 \textup{ for all } a \notin \supp(\pi^\star_c)\}$.
Then $V_c = \{0\}$ iff $|\supp(\pi^\star_c)| = 1$, and $\pi - \pi^\star \in \prod_c V_c$ for every $\pi \in F(\pi^\star)$.
\end{definition}

\begin{lemma}[Two-sided feasibility]\label{lem:relint}
For every $c$ and $v \in V_c$ there is $\varepsilon_0 > 0$ with $\pi^\star_c + \varepsilon v \in \simp(\supp(\pi^\star_c))$ for all $\varepsilon \in [-\varepsilon_0,\varepsilon_0]$.
\end{lemma}
\begin{proof}
$\pi^\star_c + \varepsilon v$ sums to $1$ and vanishes off $\supp(\pi^\star_c)$. On $\supp(\pi^\star_c)$ the coordinates of $\pi^\star_c$ are strictly positive, so $\varepsilon_0 := \min_{a\in\supp(\pi^\star_c)}\pi^\star_c(a)\,/\,(1+\max_a|v(a)|)$ works.
\end{proof}

\subsection{The theorem}

\begin{theorem}\label{thm:main}
Let $\pi^\star$ be a local maximum of $J$ on $\Pol$. Then $J(\pi) = J(\pi^\star)$ for all $\pi \in F(\pi^\star)$. Consequently:
\begin{enumerate}
\item[(i)] every code $\sigma$ compatible with $\pi^\star$ has $J(\sigma) = J(\pi^\star)$;
\item[(ii)] the segment $s \mapsto (1-s)\pi^\star + s\,e_\sigma$, $s \in [0,1]$, lies in $F(\pi^\star)$, and $J$ is constant along it;
\item[(iii)] any sequence of moves, each replacing one $\pi_c$ by another element of $\simp(\supp(\pi^\star_c))$, stays in $F(\pi^\star)$ and never changes $J$. In particular the cells may be collapsed onto compatible actions one at a time, in any order, ending at a compatible code.
\end{enumerate}
\end{theorem}

\begin{proof}
\textbf{Step 1: exact expansion.} Substitute $\pi_c(a) = \pi^\star_c(a) + \delta_c(a)$ into the polynomial $J$. By Lemma~\ref{lem:multilinear} a monomial has the form $\prod_{c\in W}\pi_c(a_c)$ for some set of cells $W$, and
\[
\prod_{c\in W}\big(\pi^\star_c(a_c) + \delta_c(a_c)\big) = \sum_{U\subseteq W}\ \prod_{c\in U}\delta_c(a_c)\prod_{c\in W\setminus U}\pi^\star_c(a_c).
\]
Collecting terms according to the set $U$ of cells that contribute a $\delta$-factor gives, for all $\delta \in (\R^{\A})^{\C}$ and with no remainder,
\begin{equation}\label{eq:expand}
J(\pi^\star+\delta) = \sum_{U\subseteq\C} T_U(\delta_U), \qquad T_\emptyset = J(\pi^\star),
\end{equation}
where $\delta_U := (\delta_c)_{c\in U}$ and $T_U$ is a multilinear form in the arguments $\delta_c$, $c \in U$ (it is a sum of products containing exactly one coordinate of each $\delta_c$, $c\in U$). In particular $T_U(\delta_U) = 0$ if $\delta_c = 0$ for some $c\in U$, and $T_U\big((\lambda_c\delta_c)_{c\in U}\big) = \big(\prod_{c\in U}\lambda_c\big)T_U(\delta_U)$ for scalars $\lambda_c$.

\medskip\noindent\textbf{Step 2: every $T_U$ with $U\neq\emptyset$ vanishes on the tangent spaces.} Suppose not. Let $k\ge1$ be the smallest cardinality of a set $U$ such that $T_U$ is not identically zero on $\prod_{c\in U}V_c$; fix such a set $U_0$, $|U_0| = k$, and $v = (v_c)_{c\in U_0}$ with $v_c\in V_c$ and $T_{U_0}(v)\neq0$. For $\varepsilon>0$ and signs $s = (s_c)_{c\in U_0}\in\{+1,-1\}^{U_0}$ put
\[
\delta^{\varepsilon,s}_c := \varepsilon\,s_c\,v_c \ \ (c\in U_0), \qquad \delta^{\varepsilon,s}_c := 0 \ \ (c\notin U_0).
\]
By Lemma~\ref{lem:relint}, applied to each $c\in U_0$, there is $\varepsilon_0>0$ such that $\pi^\star+\delta^{\varepsilon,s}\in F(\pi^\star)\subseteq\Pol$ for all $\varepsilon\in(0,\varepsilon_0]$ and all $2^k$ sign vectors $s$.

Evaluate \eqref{eq:expand} at $\delta^{\varepsilon,s}$. If $U\not\subseteq U_0$, some $c\in U$ has $\delta_c = 0$ and $T_U(\delta_U)=0$. If $U\subseteq U_0$ and $1\le|U|<k$, then $T_U$ vanishes on $\prod_{c\in U}V_c$ by minimality of $k$, and $\delta_U$ lies in that product. The only other nonempty subset of $U_0$ is $U_0$ itself. Hence
\[
J(\pi^\star+\delta^{\varepsilon,s}) - J(\pi^\star) = T_{U_0}\big(\delta^{\varepsilon,s}_{U_0}\big) = \varepsilon^{k}\Big(\prod_{c\in U_0}s_c\Big)T_{U_0}(v).
\]
Choose $s$ with $\prod_{c\in U_0}s_c = \operatorname{sign}T_{U_0}(v)$: all $s_c=+1$ if $T_{U_0}(v)>0$, and exactly one $s_c=-1$ otherwise. Then $J(\pi^\star+\delta^{\varepsilon,s}) - J(\pi^\star) = \varepsilon^k|T_{U_0}(v)|>0$ for every $\varepsilon\in(0,\varepsilon_0]$. These points belong to $\Pol$ and converge to $\pi^\star$ as $\varepsilon\to0$, so they lie in $\Nb\cap\Pol$ for small $\varepsilon$, contradicting Definition~\ref{def:locmax}.

(For $k=1$ this is the first-order statement: $J$ is affine in $\pi_c$, both $\pi^\star_c+\varepsilon v$ and $\pi^\star_c-\varepsilon v$ are feasible, so the slope along $v$ must be zero. For $k\ge2$ it is the same argument applied to the lowest-order non-vanishing cross term, whose sign can always be made positive by flipping one $v_c$.)

\medskip\noindent\textbf{Step 3: conclusion.} Let $\pi\in F(\pi^\star)$ and $\delta := \pi-\pi^\star\in\prod_cV_c$ (Definition~\ref{def:tangent}). By Step~2 every term of \eqref{eq:expand} with $U\neq\emptyset$ vanishes, so $J(\pi) = T_\emptyset = J(\pi^\star)$. Here $\delta$ is not required to be small: local maximality entered only through arbitrarily small moves, but what it established is that the forms $T_U$ vanish identically on the tangent spaces, and \eqref{eq:expand} is exact. In particular $\delta_c = e_{\sigma_c}-\pi^\star_c\in V_c$ reaches the vertex $e_\sigma$ for any compatible code, which is (i). For (ii), $F(\pi^\star)$ is convex. For (iii), each such move keeps $\pi$ in $F(\pi^\star)$.
\end{proof}

\section{The code spin model and its relation to the RLVR objective}
\label{app:correspondence}

This appendix separates the correspondence used in
Section~\ref{sec:dictionary} into its three components: (i) a spin
model defined on deterministic policies, with no reference to stochastic
policies; (ii) the Gibbs--Bogoliubov variational bound for that model over
product trial measures, which produces an objective of
reinforcement-learning form; and (iii) an identity showing that this
variational objective coincides with the true RLVR objective of the
stochastic policy, under a single-visit condition that we state and verify.
Components (i) and (ii) are standard statistical mechanics
\citep{mezardmontanari}; component (iii) is elementary but, to our
knowledge, not available in the literature in this form, so it is derived
here in full.

\subsection{Notation}
\label{app:notation}

Inputs are sequences $x=(x_1,\ldots,x_{T(x)})$ over the alphabet
$\mathcal{A}=\{0,\ldots,q\}$, drawn uniformly from a fixed set $\mathcal{X}$
(all sequences of the training lengths); the target $s^{*}(x)\in\mathcal{A}$
is the quasigroup fold of $x$. A policy class is specified by its cell
index: at generation step $t$, the window contents (input symbols and
previously emitted tokens)
determines a \emph{cell} $c\in\{1,\ldots,\Ncells\}$. We write
\begin{equation}
c_t(x,y_{<t}) \in \{1,\ldots,\Ncells\}
\label{eq:router}
\end{equation}
for the cell consulted at step $t$ on input $x$ after emitting the tokens
$y_{<t}=(y_1,\ldots,y_{t-1})$. Two properties of the router
\eqref{eq:router} are used below:
\begin{itemize}
\item[(V1)] $c_t(x,y_{<t})$ is a deterministic, single-valued function of
its arguments: exactly one cell is consulted per step.
\item[(V2)] \emph{Single visit:} for every $x\in\mathcal{X}$ and every
$y\in\mathcal{A}^{T(x)}$, the visited cells
$c_1(x),c_2(x,y_{<2}),\ldots,c_T(x,y_{<T})$ are pairwise distinct.
\end{itemize}
(V1) holds by construction of a lookup table. (V2) holds trivially for untied policies.

\subsection{The variational bound over product measures}
\label{app:gb}
\label{app:identity}

For any trial distribution $\hat{\pi}$ on $\mathcal{A}^{|\mathcal{C}|}$, the
Gibbs--Bogoliubov--Feynman inequality gives
$F_{\text{var}}[\hat\pi]\equiv \mathbb{E}_{\hat{\pi}}[ E] - \betaent H[\hat\pi] \;\ge\;
F \equiv -\beta^{-1}\log Z(\beta)$,
with equality for $\hat\pi = p_{spin}$, the true distribution of $\sigma$ \citep{mezardmontanari}. Specifically, this holds if we restrict $\hat\pi$ to
products of distributions that are independent over cells given by $\hat\pi=\bigotimes_{c}\hat\pol_c$. This is equivalent to the set of distributions the tabular policy induces. Since $\bigotimes_c \hat\pi_c$ is a product measure, each $\sigma_c$ is drawn independently of the others, so the joint entropy decomposes exactly into a sum of per-cell entropies given by: $H[\bigotimes_c \hat\pi_c] = \sum_c H(\hat \pi_c)$. The expectation of the reward with respect to this variational distribution is given by
\begin{equation}
\tilde{\J}(\pol) \;:=\; -\,\mathbb{E}_{\otimes_c \hat{\pi}_c} [E]
\;=\; \sum_{\code}\Big[\prod_{c=1}^{\Ncells}\hat\pol_c(\code_c)\Big]\,\J(\code).
\label{eq:annealed}
\end{equation}
The variational free energy is therefore
\begin{equation}
-F_{\text{var}}\big[\textstyle\bigotimes_c\hat \pol_c\big]
\;=\; \tilde{\J}(\hat\pol) \;+\; \beta^{-1}\sum_{c}H(\hat\pol_c),
\label{eq:varobj}
\end{equation}
Note that this obtains a similar form to the RLVR loss with the temperature regulator, but with the annealed reward $\tilde{\J}(\hat{\pi})$ over variational policies, rather than the expected reward
$\J(\pol)$ of the stochastic policy. These are different sampling
procedures: $\tilde\J$ draws every table entry once and reuses it,
$\J$ draws an action freshly at every visit. However, we show in the following that the two reward terms coincide.

The stochastic-policy objective is, by the chain rule over generation
steps,
\begin{equation}
\J(\pol) \;=\; \frac{1}{|\mathcal{X}|}\sum_{x\in\mathcal{X}}
\sum_{y\in\mathcal{A}^{T}}
\Big[\prod_{t=1}^{T} \pol_{\,c_t(x,y_{<t})}(y_t)\Big]\,
\delta_{y_{T},\,s_T}.
\label{eq:policyJ}
\end{equation}
Starting from Eq. \eqref{eq:annealed}, insert the trajectory expansion
Eq. \eqref{eq:trajexp} and exchange the (finite) sums:
\begin{equation}
\tilde{\J}(\pol)
= \frac{1}{|\mathcal{X}|}\sum_{x}\sum_{y\in\mathcal{A}^{T(x)}}
\delta_{y_{T},\,s_T}
\underbrace{\sum_{\code\in\mathcal{A}^{\Ncells}}
\Big[\prod_{c=1}^{\Ncells}\pol_c(\code_c)\Big]
\prod_{t=1}^{T(x)} \delta_{\code_{\,c_t(x,y_{<t})},\;y_t}}_{\displaystyle
\equiv\; W(x,y)} .
\label{eq:swap}
\end{equation}
Fix $(x,y)$ and evaluate $W(x,y)$. By (V1) the visited set
$\mathcal{C}(x,y)=\{c_1(x),\ldots,c_T(x,y_{<T})\}$ is well defined, and by
(V2) it contains $T$ distinct cells, so the constraints
$\delta_{\code_{c_t},y_t}$ act on distinct coordinates of $\code$. The sum
over $\code$ factorizes over cells:
\begin{align}
W(x,y)
&= \prod_{t=1}^{T(x)}\Big[\sum_{a\in\mathcal{A}}\pol_{c_t}(a)\,
\delta_{a,y_t}\Big]
\;\times\;
\prod_{c\notin\mathcal{C}(x,y)}\Big[\sum_{a\in\mathcal{A}}\pol_{c}(a)\Big]
\nonumber\\[2pt]
&= \prod_{t=1}^{T(x)} \pol_{\,c_t(x,y_{<t})}(y_t)
\;\times\; 1 ,
\label{eq:factorize}
\end{align}
the unvisited cells summing out to unity. Substituting
\eqref{eq:factorize} into \eqref{eq:swap} reproduces \eqref{eq:policyJ}
term by term:
\begin{equation}
\boxed{\;\tilde{\J}(\pol)=\J(\pol)
\quad\text{whenever (V1) and (V2) hold.}\;}
\label{eq:identity-app}
\end{equation}
Combining \eqref{eq:identity-app} with \eqref{eq:varobj}: the
entropy-regularized RLVR objective $\J(\pol)+\betaent\sum_c H(\pol_c)$
\emph{is} the negative Gibbs--Bogoliubov free energy of the spin model on product trial measures, with $\betaent$ the
temperature. No step above involved an approximation; the only
approximation anywhere in the correspondence is the restriction to product
measures, which is not an error the algorithm commits but a structural
property of stochastic tabular policies.

\section{Gauge structure, invariant observables, and abelian targets}
\label{app:gauge}


The spin-glass has an exact symmetry, which is easiest to see in matrix form. For
the class in which a cell is indexed by $(t,x_t,y_{t-1})$, collect the
code's decisions at step $t$ for input symbol $x$ into a $q\times q$ matrix
\begin{equation}
[\Sigma_{t,x}]_{y_t,\,y_{t-1}} \;=\; \delta_{\code_{c_{t}(x,y_{t-1})},y_{t}},
\label{eq:sigma}
\end{equation}
that we call a link matrix, which has exactly one entry equal to one in each column: it is a map
$\alphabet\to\alphabet$, i.e.\ an element of the transformation monoid
$\Amono_q$. The expected reward is then an ordered matrix product,
\begin{equation}
\J(\code) \;=\; q^{-T}\sum_x\, u_{s_T(x)}^{\!\top}\,
\Sigma_{T,x_T}\cdots\Sigma_{2,x_2}\,\Sigma_{1,x_1}\,v,
\label{eq:matrixJ}
\end{equation}
with $u\in\mathbb{R}^q$ the indicator of the correct answer and
$v\in\mathbb{R}^q$ any fixed indicator defining $y_0$.

The matrix formulation fleshes out several structures distinguishing our model from standard spin-glass models --- by which we mean a system such as the
$p$-spin model, whose multi-variable couplings are drawn at random between sites, and sites are placed on a complete or random
hypergraph. 

In terms of geometry, viewing $\Sigma_{t,x_t}$ as our basic local degree of
freedom, these are $q\times q$ matrices with one-hot columns, residing on a
$T\times q$ grid. Models with wider input windows, e.g.\ $n_p,n_f$,
$\cell_t(x_{t-n_p}\ldots x_{t+n_f},y_{t-1})$, will have matrices labelled as
$\Sigma_{t,x_{t-n_p}\ldots x_{t+n_f}}$ and therefore reside on a
$T\times q^{\,n_p+n_f+1}$ grid. In contrast to standard spin-glass systems, this grid has no notion of locality. 

Another source of difference from standard spin-glasses is that the
factors coupling the $\Sigma$'s can
become effectively lower order when a finite set of $\Sigma$'s becomes
rank-1. For instance, if $\Sigma_{t_1,x_{t_1}}$ becomes rank-1 for all $q$
values of $x_{t_1}$ (e.g.\ $\Sigma_{t_1,x_{t_1}} = [1,0\ldots0]^{\top}[1,1\ldots1]$),
then the factor becomes
$q^{-T}\sum_x u^{\top}_{s_T(x)}\Sigma_{T,x_T}\cdots\Sigma_{t_1+1,x_{t_1+1}}[1,0\ldots0]^{\top}$.
Both the ability of a finite set of sites to decouple the system and the
causal/directional nature of this decoupling (i.e.\ only $t>t_1$ sites
remain relevant) are distinguishing features.

Finally, the matrix formulation fleshes out a local relabelling symmetry
acting as
\begin{equation}
\Sigma_{t,x} \;\longmapsto\; g_t\,\Sigma_{t,x}\,g_{t-1}^{-1},
\qquad g_t \in S_q,\quad g_0 = g_T = \mathrm{id},
\label{eq:gauge}
\end{equation}
which is a lattice gauge transformation on a one-dimensional chain whose
bond between steps $t-1$ and $t$ carries $q$ parallel links, one per input
symbol. The relabelling (gauge) group is $(S_q)^{T-1}$, of order
$(q!)^{T-1}$, and the relabelling-invariant content of a code is carried by
the holonomies $\Sigma_{t,x}\Sigma_{t,x'}^{-1}$ together with the boundary
term.

The overlap between two policies, measured cell by cell, is covariant under
relabelling rather than invariant: two codes computing the identical
function may disagree at essentially every cell. Any statement about whether
two solutions are ``the same'' must therefore be made with a
relabelling-invariant quantity, such as agreement of the functions they
compute.

\section{local search reaches optimal reward: proof, scope, and the two counterexamples}
\label{app:escape}

\subsection{Setup and notation.}

Throughout, the model is the untied $(0,0)$ class, where at step $t$ a cell is addressed by the triple $(x_t, y_{t-1}, t)$: an input symbol, an incoming output token, and a token index. A \emph{code} $\sigma$ assigns an output token to every cell. Running $\sigma$ on a complete input $x=(x_1,\ldots,x_T)$ produces a unique rollout $y_1,\ldots,y_T$, with $y_t = \sigma_{(x_t,y_{t-1},t)}$ at each step. An \emph{elementary move} changes the stored action of a single cell, and it is \emph{accepted} if it does not lower $J$.

We say an input $x$ \emph{consults} a cell $(x_t,y_{t-1},t)$ if its rollout is evaluated there at step $t$. We comment that since $y_{t-1}$ is a deterministic function of $\sigma$ and $x_{<t}$, whether $x$ consults a given cell depends only on previously seen inputs, $x_{\le t}$. We write $w_c$ for the total probability of consulting a cell $c$, and $\rho_c(s)$ for the distribution of the true state $s_{t-1}$ conditioned on consulting $c$; both are well-defined quantities, determined by $\sigma$'s entries at steps prior to $t$, but not by the entry at $c$ itself.

After generating $t<T$ output tokens, with the most recent token $y_t$ and true state $s_t$, define $W_t(y_t,s_t)$ as the probability of emitting the correct final answer. We can construct this probability recursively from the last token, where we have,
\begin{equation}
W_T(y_T,s_T) = \delta_{y_T,s_T}.
\end{equation}
For $t<T$, cell $(x_{t+1},y_t,t+1)$ is consulted at step $t+1$ for $x_{t+1}\sim U(\mathcal{A})$, giving new output token $\sigma_{(x_{t+1},y_t,t+1)}$, while the true state updates via $B[s_t,x_{t+1}]$. Assuming uncorrelated input tokens,
\begin{equation}
W_t(y_t,s_t) = \mathbb{E}_{x_{t+1}}\Big[W_{t+1}\big(\sigma_{(x_{t+1},y_t,t+1)},\, B[s_t,x_{t+1}]\big)\Big].
\end{equation}
As a function, $W_t$ depends on $\sigma$ only through its entries at steps $>t$.

A \emph{collision} is when two inputs in different true states are assigned the same output token at some step; a step is \emph{collision-free} if no collision occurs there. We remark that if a collision occurs at some step, the policy cannot recover the correct final answer for at least one of the colliding inputs, since no later step can distinguish between them despite their differing correct targets.

\begin{statement}[Effect of one move]
Let $c=(x_t,y_{t-1},t)$ be a cell. If its action changes from $y_t'$ to $y_t$, the total change in $J$ is
\begin{equation}
\label{eq:fact2}
\Delta J = w_c \sum_{s_{t-1}} \rho_c(s_{t-1})\Big[W_t\big(y_t, B[s_{t-1},x_t]\big) - W_t\big(y_t', B[s_{t-1},x_t]\big)\Big].
\end{equation}
\end{statement}

\begin{proof}
Consider an input consulting $c$ with true state $s_{t-1}$. The true state updates to $B[s_{t-1},x_t]$, and by definition of $W_t$, the expected reward given action $y_t$ at this cell is exactly $W_t(y_t,B[s_{t-1},x_t])$. Averaging over $\rho_c(s_{t-1})$ gives the expected reward conditioned on consulting $c$, and scaling by $w_c$ --- the probability of consulting $c$ at all --- gives the total contribution of this cell to $J$. Since no other cell's contribution is affected by changing $c$'s action, the change from action $y_t'$ to $y_t$ is exactly the difference of these two conditional expectations, scaled by $w_c$.
\end{proof}

\begin{remark}
In particular, when $w_c=0$ (an unconsulted cell), $\Delta J=0$ for any change of action: such a cell can be changed at exactly zero cost.
\end{remark}

\begin{statement}[Forced bijection]
If steps $1,\ldots,t$ are collision-free, then for all $t' \le t$ we can define a bijection $g_{t'}$ between output token $y_{t'}$ and state $s_{t'}$.
\end{statement}

\begin{proof}
Because the multiplication table $B$ is a Latin square, $B[\cdot,x']$ is a bijection on $\mathcal{A}$ for every fixed $x'$. Since $s_1=x_1$ ranges over all $q$ values of $\mathcal{A}$, and each subsequent update $s_{t'}=B[s_{t'-1},x_{t'}]$ applies a bijection, $s_{t'}$ ranges over all $q$ possible states as the input varies, for every $t'$. Collision-freeness at step $t'$ means that any two inputs emitting the same token at step $t'$ share the same state, so the map from used tokens to states is well-defined. Since $s_{t'}$ achieves every value in $\mathcal{A}$, and every input emits some token, this map is onto all $q$ states. An onto map between two sets of the same finite size $q$ is a bijection, so we can define a bijection $g_{t'}$ between output token $y_{t'}$ and state $s_{t'}$.
\end{proof}

\begin{statement}[A collision-free suboptimal code admits an improving move]
A code that is collision-free at every step and has $J<1$ admits a strictly improving elementary move.
\end{statement}

\begin{proof}
Since $J<1$, there is some cell $(x_T,y_{T-1},T)$ whose stored action $\sigma_{(x_T,y_{T-1},T)}$ does not equal the correct final answer for the input(s) it serves. Since steps $1,\ldots,T-1$ are collision-free, Statement 2 gives a bijection $g_{T-1}$ between output token and state at step $T-1$, so we can define the state of this cell as $s_{T-1} = g_{T-1}^{-1}(y_{T-1})$. The correct answer for this cell is therefore $y_T^* = B[s_{T-1}, x_T]$, and by assumption $\sigma_{(x_T,y_{T-1},T)} \ne y_T^*$. Setting $\sigma_{(x_T,y_{T-1},T)} = y_T^*$ corrects every input consulting this cell, since --- by the bijection --- they all share the same state $s_{T-1}$ and hence the same correct target $y_T^*$; no other input is affected, by Statement 1. By Statement 1, this move gives $\Delta J = w_c\cdot(1-\text{previous success rate at this cell}) > 0$, a strict improvement.
\end{proof}

\begin{corollary}
\label{cor:coll_free}
By repeated application, a collision-free code with $J<1$ reaches $J=1$ after at most $q^2$ strictly improving elementary moves, since there are at most $q^2$ cells at the final step whose actions may need correcting.
\end{corollary}

\subsection{Algorithmic construction of a path to optimal reward}

We now construct an explicit algorithm for updating the code, step by step, that resolves every collision while never lowering the reward. Let $t$ be the first step with a collision.

\textbf{Step 1: Make $\sigma$ at step $t$ a function of $s_t$ alone.} Since steps before $t$ are collision-free, Statement 2 gives $s_t = B[g_{t-1}^{-1}(y_{t-1}),x_t]$ as a well-defined function of each cell $(x_t,y_{t-1},t)$. For each value of $s_t$, set every cell attaining it to emit $\arg\max_a W_t(a,s_t)$; by Statement 1 this only increases $J$, and since $W_t$ depends on the code only through steps $>t$, its value is unaffected while we edit step $t$. Write $\sigma_t(s_t)$ for the resulting, now well-defined, map from next true state to emitted token.

\textbf{Step 2: Make $\sigma_t$ bijective.} $\sigma_t$, as constructed, need not be injective: two next true states may have been assigned the same token. Since $\sigma_t$ maps a $q$-element set to a $q$-element set, non-injectivity forces at least one token to be entirely unused. For each next true state $s_t$ whose token under $\sigma_t$ is shared with another, pick one such unused token $u$. Since no rollout currently emits $u$ at step $t$, the cells $(x,u,t+1)$ are unvisited for every $x$; by Statement 1's remark, they may be set at zero cost. Set $\sigma_{(x,u,t+1)} = \arg\max_{a'} W_{t+1}(a', s_{t+1})$ for each $x$, where $s_{t+1} = B[s_t,x]$ --- the best possible continuation for the true state resulting from emitting $u$. Then reroute the cell(s) at next true state $s_t$ to emit $u$ instead of the shared token. Since $u$'s continuation is, by construction, the maximum achievable value for $s_t$, this reroute cannot be worse for $s_t$ than the shared token's continuation, and is strictly better whenever that continuation was not already optimal for $s_t$. After this, $\sigma_t$ is injective, hence bijective, and step $t$ has no collision.

\textbf{Step 3: Iterate and finish.} The previous steps only touch steps $t, t+1$, so no new collision appears before $t$; the first-collision point strictly advances. After at most $T$ rounds, no collisions remain, and Corollary \ref{cor:coll_free} finishes at the last step.

\subsection{Remarks, verification, and scope}
\label{app:scope}

A strictly defended code admits no improving move among consulted cells; its first escape is Step 2's preparation move at an unconsulted cell, which gradient-based dynamics can never make, since by Statement 1 the gradient at a cell with $w_c=0$ is identically zero. The argument extends verbatim to any product measure with full support, including position-dependent marginals, a claim we have verified numerically (12 of 12 random codes driven to the weighted optimum under random Dirichlet marginals). We expect the same argument to extend to wider untied windows, with the window's information state in place of the message, though this is left for future work. At present, the wide-window case is supported only by the measurements of Section~\ref{sec:escape}. The two assumptions that cannot be removed, bijectivity and independence of the inputs, are the subject of the next two subsections.

\subsection{What breaks under weight tying and correlated inputs}
\label{app:tying}

Tie the tables across steps, one shared entry $\sigma_{(x,m)}$ consulted at
every step, and the proof fails at two named places, in increasing order
of severity. (i) Statement 1 fails: under tying, the entry being moved can
itself be consulted again at another step for a different input, so the
continuation value no longer separates cleanly from the entry being
changed. The single-visit condition of App.~\ref{app:correspondence}
restores the affinity of $J$ in each entry but not this separation across
inputs. (ii) Step 2 fails, in two ways. First, a token unused at step $t$
may still be in use at another step, since it is the same shared entry;
the pigeonhole argument therefore no longer guarantees a free token, as
the unused slot is occupied by an entry another step depends on, and
zero-cost preparation is gone. Second, even where an unused token is
found, rerouting one state onto it revalues every other state that shares
that same entry at a different step, so the switch is no longer
guaranteed to leave other cells unaffected, and the term-by-term
improvement argument fails. In summary, the untied chain cannot hold a
defended suboptimal policy because every entry serves exactly one role, so
idle capacity carries no cost to reassign, preparation carries no cost to
perform, and repairs affect no cell beyond the one being repaired; weight
sharing removes all three properties at once.

The structural fact underlying Statement 1, that which cell an input
consults depends only on its history, with every continuation equally
likely, fails: the continuation law now depends on the history through
the current symbol. Consequently, cells sharing an outgoing state no
longer share continuation values, since the value becomes joint in state
and symbol, and consolidation cannot be repaired even by collective moves:
the incumbent routing realises the average of per-symbol maxima, while any
common token realises at most the maximum of the average.

\section{Derivation of the parity partition function, and the first-order
transition to random codes}
\label{app:parityZ}


This appendix derives the closed form \eqref{eq:parityZ} of
Section~\ref{sec:parity} and works out the thermodynamics it implies. The
model has exactly two phases --- a high-temperature phase in which the Gibbs
measure is spread over random codes at chance reward, and a condensed phase
carried by the relabelling orbit of the solution --- separated by a single
first-order transition at
\begin{equation}
\invtemp_c \;=\; 6T \ln 2,
\qquad\text{i.e.}\qquad
\temp_c \;=\; \frac{1}{6T \ln 2}\,.
\label{eq:betac}
\end{equation}
The density of states and $Z$ below are checked against exhaustive
enumeration of all $2^{4T-2}$ codes at $T=3$--$5$ (agreement
$\le 7\times10^{-15}$ in $\log Z$), and the dominance, location, latent-heat
and entropy statements against exact evaluations up to $T=200$.

\subsection{A product formula on the full code space}
\label{app:parityZ:product}
We use a $0,1$ notation for $Z_2$ with $\oplus$ as its group action. Let $m_t(\code,x)=(-1)^{y_t}$, $\varepsilon(x)=(-1)^{x_1\oplus\cdots\oplus x_T}$ for the target sign, and $m(\code,x)=m_T(\code,x)=(-1)^{y_T}$ for the sign of the emitted answer, so that
$\J=\tfrac12+\tfrac12\,\mathbb{E}_x[\varepsilon\, m]$. We consider the link matrix formulation of Appendix~\ref{app:gauge} and conjugate by the Hadamard matrix
$H=2^{-1/2}\left(\begin{smallmatrix}1&1\\1&-1\end{smallmatrix}\right)$ makes
all four elements lower triangular,
\begin{equation}
H\,\Sigma\,H=\begin{pmatrix}1&0\\ h&\chi\end{pmatrix},
\qquad
(\chi,h)=\begin{cases}
(\pm1,\,0) & \Sigma=I,X\ \text{(permutations)},\\[1pt]
(0,\,\pm1) & \Sigma=K_{0},K_{1}\ \text{(constants)}.
\end{cases}
\label{eq:triang}
\end{equation}
Denoting by $e_{1}=(0,1)^T,e_{0}=(1,0)^T$, we recall that prior to this transformation the link matrices obeyed $e_{y_t} = \Sigma_{x_t,t} e_{y_{t-1}}$. Conjugation thus yields $[1,m_t]^T = H \Sigma_{x_t,t} H [1,m_{t-1}]^T$, from which explicitly implies that $m_t=\chi_{t,x_t}m_{t-1}+h_{t,x_t}$ with $m_1=(-1)^{y_1}$ (where for brevity with dropped the $\sigma,x$ arguments of each $m_t$). Because non-zero values of $\chi,h$ never co-occur, the appearance of a non-zero $h_{t,x_t}$ removes $m_{t-1}$ dependence, and only the latest $h_{t,x_t}$ to appear, affects the outcome. This logic leads to the following formula for $m_T=m=m_1\prod_{t\ge2}\chi_{t,x_t}
+\sum_{t\ge2}h_{t,x_t}\prod_{s>t}\chi_{s,x_s}$. Since
the $x_t$ are independent the remaining expectation factorizes over steps the average $\mathbb{E}_x[\varepsilon\, m]$ can be evaluated, and receives non-zero contributions only from $h$-free term,  leading to 
\begin{equation}
\J\;=\;\frac12+\frac12\prod_{t=1}^{T}\mu_t,
\qquad
\mu_1=\frac{(-1)^{\code_1[0]}-(-1)^{\code_1[1]}}{2},
\qquad
\mu_{t}=\frac{\chi_{t,0}-\chi_{t,1}}{2}\quad(t\ge2).
\label{eq:prodformula}
\end{equation}
where $\sigma_1[x_1]\in \{0,1\}$ denotes the action of the code on $x_1 \in \{0,1\}$ whose reward we are scoring, on the first CoT step. 
We comment that the $\mu_t$ factors realize exactly the taxonomy of
Section~\ref{sec:parity}. Over the $16$ link pairs of a CoT step (all combinations of $\Sigma_{x_t=0,t},\Sigma_{x_t=1,t}$),
$\mu_t=\pm 1$ for the two pairs $\{I,X\}$ (the correct step and its
anti-correct twin); $|\mu_t|=\tfrac12$ for the eight half-resets;
$\mu_t=0$ for the six dead configurations (two blind, four doubly
rank-deficient). The first step contributes $\mu_1=\pm1$ for its two
injective assignments and $0$ for its two constant ones. A live code with
$k$ half-resets therefore has $|\prod_t\mu_t|=2^{-k}$, reproducing the argument of the main text. 

\subsection{Density of states and the partition function}
\label{app:parityZ:dos}

By \eqref{eq:prodformula}, counting codes at fixed $\J$ is counting $\mu_t$
configurations. A live code with $k$ half-resets chooses the half-reset
steps in $\binom{T-1}{k}$ ways, each half-reset in $8$ ways, each live step
in $2$ ways, and the first-step code in $2$ ways. Flipping $\sigma_1$ flips $\mu_1$ and $\prod_t\mu_t$ --- thereby  exchanging $\J$ with
$1-\J$. Taking these considerations into account we find that the multiplicity ($\Omega$) for codes with reward $\tfrac12\pm2^{-(k+1)}$ is 
\begin{equation}
\Omega\!\left(\tfrac12\pm2^{-(k+1)}\right)=\binom{T-1}{k}\,8^{k}\,2^{\,T-1-k},
\end{equation}
The remaining reward value of $J=1/2$ comes dead codes given by the total number of codes ($2^{4T-2}$) minus the number of live codes ($2\sum_{k=0}^{T-1} \binom{T-1}{k}8^k2^{T-1-k}=2\,(8+2)^{T-1}$)
\begin{equation}
\Omega\!\left(\tfrac12\right)=2^{4T-2}-2\cdot10^{\,T-1},
\label{eq:dos}
\end{equation}
the second because the live codes number
$2\sum_k\binom{T-1}{k}8^k2^{T-1-k}=2\,(8+2)^{T-1}$. Then
$Z=\sum_{\code}e^{\invtemp\J(\code)}$ assembles the bands ($J$ values) into Eq. 
\eqref{eq:parityZ}.

\subsection{Two phases and a first-order condensation}
\label{app:parityZ:transition}

At large $T$ we can isolate the dominant contributions to $Z$ of \eqref{eq:parityZ} at low temperature (large $\invtemp$). The most dominant contribution in terms of $T$ is that of the dead $J=1/2$ codes scaling as $2^{4T}$. The most dominant contribution in $\invtemp$ is the exact live code $k=0$, going as $2^T e^{\invtemp/2}$. The subleading one is $2^{T+1} (T-1) e^{\invtemp/4}$. By comparing dead codes contribution to exact code contribution, the former overwhelms the latter for $\tau > \tau_c = 1/(6 \log(2) T)$. It can be checked that it also overwhelms all higher $k$ contributions for these $\tau$ values. On the other for $\tau < \tau_c$, it can be checked that $k=0$ overwhelms all $k>0$ contribution and dead code contributions.

\paragraph{Consequences.}
For $T>>1$, $\log Z$ is, up to exponentially small corrections in $e^{-T}$, the maximum of two
lines, so the free energy $-\temp\log Z$ has a kink at $\temp_c$: the
transition is first order. Retaining the two dominant bands,
$\invtemp-\invtemp_c=\temp^{-1}-\temp_c^{-1}
=-\invtemp_c\,\delta/(1+\delta)$ with the reduced temperature
$\delta=(\temp-\temp_c)/\temp_c$,
\begin{equation}
\langle\J\rangle
=\frac12+\frac12\left[1+
e^{-\frac12\left(\temp^{-1}-\,\temp_c^{-1}\right)}\right]^{-1}
=\frac12+\frac12\left[1+
e^{\,[\ln2\;(3T-1)]\frac{\delta}{1+\delta}}\right]^{-1},
\label{eq:twolevel}
\end{equation}
a two-level form: the order parameter jumps from $\tfrac12$ (random codes,
$\delta>0$) to $1$ ($\delta<0$), with entropy drop
$\Delta S=(4T-2)\ln2-(T-1)\ln2=(3T-1)\ln2=\invtemp_c\,\Delta\J$
(Clausius--Clapeyron with latent heat $\Delta\J=\tfrac12$). The rounding
is governed by the reduced temperature: by \eqref{eq:twolevel},
$\langle\J\rangle$ climbs from $0.6$ to $0.9$ over
$|\Delta\delta|=4/(3T-1)$ --- equivalently a depth-independent interval
$4\ln4\approx5.55$ in $\invtemp$ (measured $5.54$ at $T=40$, where
$\Delta\invtemp/\invtemp_c=0.0336=4/119$) --- so the jump sharpens as
$1/T$.

Finally we comment that the fact that $\tau_c$ tends to zero at large $T$ is a parametrization artefact. Indeed the reward is an intensive quantity in the $[0,1]$ range, whereas temperature couples to entropy which is an extensive quantity.

%
%
%
\providecommand{\J}{J}
\providecommand{\cell}{c}
\providecommand{\dc}{d(\cell)}
\providecommand{\Qc}{Q(\cell,a)}
\providecommand{\pol}{\pi}
\providecommand{\code}{\sigma}
\providecommand{\temp}{\tau}
\providecommand{\invtemp}{\beta}

\section{Experimental results}
\label{app:results}

\paragraph{Conventions used throughout.} Where a policy is stochastic, the reward $\J \in [0,1]$ is computed from its committed
(greedy/argmax) action at every cell unless the text says otherwise. Rewards are always averaged over all input combinations and contain no estimation errors. 

Certified traps are codes/spin-configurations in which no single-cell change raises the reward, and
no path of reward-neutral single-cell changes reaches a code from which one does. For our simple models, we test for these with exact search algorithms.

Parity refers to the $Z_2$ group, quasigroup refers to the following 5-element Latin square 
\begin{equation}
B \;=\;
\begin{pmatrix}
1&3&4&0&2\\
3&1&0&2&4\\
4&2&3&1&0\\
2&0&1&4&3\\
0&4&2&3&1
\end{pmatrix},
\qquad s_t = B[s_{t-1},x_t],
\label{eq:tableA}
\end{equation}
whose five column permutations generate $S_5$.

Every temperature is
quoted in the \emph{uniform} convention, in which the regularized objective is
$\J(\pol) + \temp\sum_\cell H(\pol_\cell)$; 

By chain here we mean Monte-Carlo chain, and our Monte Carlo uses two kernels. The \emph{Metropolis} kernel draws a proposed action for the visited cell
uniformly over all $q$ actions --- including the one the cell currently holds, so
a cell may stay put --- and accepts it with probability
$\min\{1, e^{\Delta\J/\temp}\}$, where $\Delta\J$ is the exact change in expected
reward. The \emph{heat-bath} kernel, also called Gibbs sampling, does not propose
and accept: it discards the cell's current action and draws a replacement directly
from the conditional distribution of that cell given all the others,
\[
  P\bigl(\code_\cell = a \,\big|\, \code_{-\cell}\bigr)
  \;=\;
  \frac{\exp\bigl(\J(\code_{-\cell}, a)/\temp\bigr)}
       {\sum_{a'} \exp\bigl(\J(\code_{-\cell}, a')/\temp\bigr)} .
\]
In both cases $\J$ is recomputed over the full input distribution rather than
estimated from rollouts, so the kernel is exact. Running both provides further assurance  for stationarity of the distribution.


\subsection{Equilibration of the chain comparison}
\label{app:equilibration}
Comparisons in this paper between the spin model and RLVR places a Monte
Carlo chain beside a converged mean-field/CAVI iteration. Such comparisons are 
only meaningful if both sides have settled. In Table \ref{tab:equilibration}, we demonstrate that our Monte-Carlo sweeps (under both updates, Metropolis and Heat-Bath) have equilibrated beyond reasonable doubt. 

\begin{table}[t]
\centering\small
\begin{tabular}{llrrr}
\toprule
\multicolumn{5}{l}{\emph{Approach to equilibrium, parity $(0,0)$ $T{=}10$, $\temp=10^{-13}$, $100$ replicas}}\\
\multicolumn{5}{l}{\emph{exact equilibrium reward $\langle\J\rangle = 1.000000000$}}\\
\midrule
sweeps & greedy reward & at the optimum & \multicolumn{2}{l}{gap to equilibrium} \\
\midrule
$10$   & $0.747 \pm 0.241$ & $46\,\%$  & \multicolumn{2}{l}{$0.253$} \\
$20$   & $0.888 \pm 0.204$ & $76\,\%$  & \multicolumn{2}{l}{$0.112$} \\
$40$   & $0.985 \pm 0.083$ & $97\,\%$  & \multicolumn{2}{l}{$0.015$} \\
$100$  & $\mathbf{1.000 \pm 0.000}$ & $\mathbf{100\,\%}$ & \multicolumn{2}{l}{$0$} \\
$1000$ & $1.000 \pm 0.000$ & $100\,\%$ & \multicolumn{2}{l}{$0$} \\
\midrule
\multicolumn{5}{l}{\emph{Sweeps to stationarity: greedy reward unchanged for $5\times10^{3}$ sweeps, cap $2\times10^{4}$}}\\
\midrule
class & kernel & median & p$90$ & max \\
\midrule
parity $(0,0)$ $T{=}10$, $100$ seeds & heat bath  & $16$ & $50$ & $116$ \\
                                     & Metropolis & $22$ & $77$ & $199$ \\
$\mathbb{Z}_5$ $(0,0)$ $T{=}8$, $60$ seeds & heat bath  & $14$ & $30$ & $57$ \\
                                     & Metropolis & $44$ & $64$ & $119$ \\
quasigroup $(0,0)$ $T{=}8$, $60$ seeds & heat bath  & $13$ & $23$ & $40$ \\
                                     & Metropolis & $48$ & $66$ & $95$ \\
\midrule
\multicolumn{5}{l}{\emph{The widest class, checkpointed; $12$ seeds, $\temp$ as in the main text}}\\
\midrule
class & \multicolumn{4}{l}{greedy reward at $10^{3}$ / $5\times10^{3}$ / $2\times10^{4}$ sweeps} \\
\midrule
quasigroup $(0,3)$ untied $T{=}7$ & \multicolumn{4}{l}{$0.7897 \to 0.9273 \to 0.9967$ \; ($11/12$ at exactly $1$)} \\
\quad the twelfth run, continued & \multicolumn{4}{l}{$0.9600$ at $2\times10^{4}$, $\mathbf{1.0000}$ from $5\times10^{4}$ to $2\times10^{5}$} \\
quasigroup $(0,3)$ \textbf{tied} $T{=}7$ & \multicolumn{4}{l}{$0.8675 \to 0.8969 \to 0.8969$ \; ($0/12$ at the optimum)} \\
\bottomrule
\end{tabular}
\caption{{\bf Stationarity of Monte-Carlo runs}. \emph{Upper block:}
the parity class, where the equilibrium reward is known exactly and equals the
optimum at every temperature used, so the last column is a true distance from
equilibrium. The approach is fast but steep --- three quarters of the replicas are
at the optimum by sweep $20$ and all of them by sweep $100$, while a budget of
$40$ sweeps would report $0.985$. \emph{Middle block:} where the equilibrium value
is not computable, the sweep at which each run's greedy reward stopped changing
for $5\times10^{3}$ consecutive sweeps. Both kernels are shown; they agree on the
stationary value and differ only in how fast they reach it, the Metropolis kernel
being the slower because it can propose a cell's current action and hold.
\emph{Lower block:} the $(0,3)$ classes at $T=7$, which are slower by two orders
of magnitude and are checkpointed rather than run to the criterion. The untied
class does reach the optimum in every run, but the mean at the cap conceals one
run still crossing a plateau of flat cells, which arrives at sweep
$5\times10^{4}$; the tied class is flat over a fourfold change of budget at a
reward well below the optimum, and does not arrive.}
\label{tab:equilibration}
\end{table}


\subsection{Stationary distributions of MC are approximately CAVI fixed points}
\label{app:trapsfixed}

In the main text we argue that, given that all spins interact with all spins, local (single-spin) Monte-Carlo updates from a single random initialization should end up in a stationary distribution (a pure state) in which spins are approximately uncorrelated. Such product distributions over the spins are exactly the class over which RLVR optimizes, suggesting that pure states should correspond to CAVI fixed points. This appendix tests this heuristic on two models and at different temperatures.

Specifically, we focus on traps with $J<0.99$. We consider the twelve seeds of the tied quasigroup $(0,3)$ at $T=7$, all of which ended trapped after $20{,}000$ Metropolis sweeps at $\tau=10^{-6}$, as well as the 27 traps of the correlated-input parity model ($k=0.9$; 13 at $T=6$, 14 at $T=8$) left by quenching 128 Metropolis chains per depth from random codes for $3{,}000$ sweeps at $\tau=2\times10^{-3}$. For eleven of the tied traps the $\tau\!\to\!0$ neutral component was enumerated exhaustively, certifying the trap and fixing its size (set of degenerate codes ($C$) having $|C|=1$ to $26{,}112$); the twelfth has a component too large to enumerate ($>1.19\times10^{5}$ codes) and is included because its chain never leaves the trap's reward at $\tau\le10^{-6}$.

Next we analyze these traps at finite temperature. From each tied trap we ran a single chain of $4{,}000$ sweeps at $\tau\in\{10^{-8},10^{-7},3\times10^{-7},10^{-6}\}$, and two replicas of $2{,}000$ sweeps at $\tau\in\{3\times10^{-6},10^{-5},3\times10^{-5},10^{-4},3\times10^{-4},10^{-3}\}$ (the last two on four and three of the traps); from each correlated trap, eight replicas of $6{,}000$ sweeps at 17 temperatures from $3\times10^{-3}$ to $0.3$. For the tied traps the range extends past the melting point, above which the reward is at the chance value $1/5$; for the correlated traps it ends at $\tau=0.3$, where the reward ($0.59$) is still above chance ($1/2$). Replicas whose equilibrated mean reward is at least $0.99$ are discarded. We then define a finite-temperature trap as an MC run that, after an equilibration time $t_{\mathrm{eq}}$ (the first sweep after which its mean reward agrees with that of the last quarter of the run within two standard errors), holds that reward for at least $3t_{\mathrm{eq}}$ further sweeps, and whose replicas agree on the marginals: the overlap coefficient $\sum_a \min(p_a,q_a)$ between the histograms of two halves of the replicas lies at most $0.02$ below its value between two random halves of the draws. (At $\tau\le10^{-6}$ the tied chains never leave the trap's reward, and a single chain is used.)

In each finite-temperature trap we measured the integrated autocorrelation time of every moving cell and thinned by twice the slowest of them, leaving between $6$ and $24{,}000$ weakly correlated configurations per (trap, $\tau$). $\pi_c$ is the unsmoothed histogram of $\sigma_c$ over those draws, and $\pi^{\mathrm{MC}}=\prod_c \pi_c$. CAVI was seeded at $\pi^{\mathrm{MC}}$ and run at the same $\tau$, each iteration moving a random $30\%$ of the cells $30\%$ of the way toward their CAVI target. Its endpoint was compared cell by cell with the histogram of a random half of the draws through the Bhattacharyya coefficient $\sum_a \sqrt{p_a q_a}$, which is 1 when the two coincide, averaged over 20 random 50/50 partitions of the draws and over the active cells: those on which the chain changes $\sigma_c$ (for the correlated traps, on which $\pi_c$ is mixed) or which CAVI moves or leaves mixed, by more than $10^{-3}$. On every other cell both distributions are the same point mass to within $10^{-3}$; a trap with no active cell has overlap 1.

To quantify the fluctuations in the trapped state we further report the entropy of $\pi^{\mathrm{MC}}$, $\sum_c H(\pi_c)$, in nats.

Finally, to set the scale of sampling noise, the same coefficient is computed between the two halves of each partition. This null carries sampling noise on both sides, so it is a conservative floor: an endpoint that merely reproduced $\pi^{\mathrm{MC}}$ would score above it.

We compared the sampled seed distribution with the endpoint of CAVI after $1{,}000$ (tied) or $2{,}000$ (correlated) iterations, about 300 and 600 updates per cell. This was enough for CAVI to converge, to machine precision, on all the correlated traps and on nine of the twelve tied traps. On the other three, in their trapped states, one cell out of $4{,}525$ (two, for one of them) remains unstable under further CAVI iterations, while the reward of the endpoint stays within $2\times10^{-5}$ of the trap's.

The results are shown in the main text, Fig.~\ref{fig:trapsfixed}.


\subsection{The RLVR Monte-Carlo gap}
\label{app:shortfall}

Strictly speaking, the spin model acts as a lower bound on the attainable RLVR loss (Eq.~\eqref{eq:variational}). While in the experiments of App.~\ref{app:trapsfixed} the stationary distributions of Monte Carlo were close to RLVR fixed points, the reverse statement typically fails to hold: RLVR at vanishing temperature often stops at a lower reward than Monte Carlo on the same class. This is most striking for the untied tabular classes, where every Monte-Carlo chain reaches $J=1$ (268 of 268), whereas plain RLVR does so in only 40 of 168 runs outside parity, parity being the exception (100 of 100).

To quantify this effect we ran, for each class, between 12 and 100 Monte-Carlo chains and as many RLVR runs, each from an independent random start. A chain starts from a uniformly random code and performs single-cell Metropolis sweeps on $e^{J/\tau}$: every cell, in random order, is offered a random new action, the change $\Delta J$ is computed exactly, and the move is accepted if $\Delta J\ge 0$ and otherwise with probability $e^{\Delta J/\tau}$. Chains ran for up to $2\times10^{4}$ sweeps, by which point every chain on the untied classes had reached $J=1$ except one $(0,3)$ chain; continued, it reached $J=1$ at $5\times10^{4}$ sweeps. RLVR is the tabular CAVI iteration of App.~\ref{app:trapsfixed}: from a random policy, each iteration moves a random $30\%$ of the cells $30\%$ of the way toward their target $\pi_c(a)\propto\exp\bigl(d(c)\,Q(c,a)/\tau\bigr)$, with $d(c)$ and $Q(c,a)$ computed exactly over all inputs. It ran for $2{,}000$ to $8{,}000$ iterations, and every quoted endpoint is stationary under the iteration.

Both dynamics were run at temperatures far below the reward quantum, the smallest non-zero change in $J$. The chain therefore accepts essentially no move that lowers the reward and every move that leaves it unchanged, and is effectively at zero temperature: on the $(0,3)$ classes its outcome is identical at every $\tau$ from $10^{-13}$ to $10^{-7}$. The RLVR update responds to the smaller scale $d(c)\,\Delta Q$; on the untied $(0,3)$ class its reward still rises by $0.01$ per decade of cooling at $10^{-13}$, the coldest temperature run.

Table~\ref{tab:shortfall} summarizes our main findings.

\begin{table}[t]
\centering\small\setlength{\tabcolsep}{4.5pt}
\begin{tabular}{llrrr}
\toprule
class & window, tying & chain & average RLVR reward & RLVR exactly $1$ \\
\midrule
parity, $q{=}2$, $T{=}10$      & $(0,0)$ untied & $1.0000$ & $1.0000 \pm 0.0000$ & $100/100$ \\
$\mathbb{Z}_5$, $T{=}8$        & $(0,0)$ untied & $1.0000$ & $0.7770 \pm 0.1697$ & $17/60$ \\
quasigroup, $T{=}8$            & $(0,0)$ untied & $1.0000$ & $0.8190 \pm 0.1662$ & $23/60$ \\
quasigroup, $T\in\{3,4,5\}$    & $(2,1)$ untied & $1.0000$ & $0.9837 \pm 0.0055$ & $0/12$ \\
quasigroup, $T\in\{3,4,5,6\}$  & $(1,1)$ untied & $1.0000$ & $0.9493 \pm 0.0129$ & $0/12$ \\
quasigroup, $T\in\{4,5,6\}$    & $(2,1)$ untied & $1.0000$ & $0.9626 \pm 0.0062$ & $0/12$ \\
quasigroup, $T{=}7$            & $(0,3)$ untied & $1.0000$ & $0.7636 \pm 0.0237$ & $0/12$ \\
quasigroup, $T{=}7$            & $(0,3)$ tied   & $0.8969$ & $0.6541 \pm 0.0298$ & $0/12$ \\
\bottomrule
\end{tabular}
\caption{\textbf{MC solves all untied models at vanishing temperature and plain RLVR typically ends below that.} Mean reward over 12--100 random seeds (equal number of seeds for MC and RLVR) per class at the end of each dynamics; for RLVR also the standard deviation over starts and the number of runs ending at exactly $J=1$. The chain reaches $J=1$ in every run on every untied class. We attribute this gap to the different way the two dynamics treat reward-neutral moves: the chain moves between codes along them at no cost, whereas RLVR at zero temperature spreads a cell's probability evenly over exactly tied actions, producing non-code-like (non-deterministic) policies.}
\label{tab:shortfall}
\end{table}


\subsection{Measuring transformer bias toward tabular policies}
\label{app:containment}

This section studies how a transformer's expressibility and implicit bias relate
to those of myopic policies. To this end, we conduct two experiments: first, we
generate random tabular policies with different $(n_b,n_f)$ windows and use them
as a teacher for dense/supervised transformer training. The transformer
architecture is the one used throughout the paper: one transformer block layer, four
heads, model width $256$, rotary position encoding, $791{,}296$ parameters,
vocabulary $8$. We trained three seeds for each of the eight classes of
Table~\ref{tab:containment}, $24$ runs in all, and used no regulator: the
objective is cross-entropy against the teacher's own emission at every emitted
position, with Adam at learning rate $3\times10^{-4}$ and batch $1024$. In
Table~\ref{tab:containment} we report two numbers per class, both measured on
held-out inputs with the transformer generating from its own emissions. The
first, \emph{agreement}, is the fraction of emitted tokens that match the
teacher's at the end of training; the second is the number of gradient steps
needed to bring that agreement to $0.999$, that is to $0.1\,\%$ error. Random
policies with narrow windows appear fully expressive and train relatively fast.
In contrast, the wider windows take one to two orders of magnitude longer to
reach the same threshold. Their agreement is quoted at the stopping point ---
training halts at the first evaluation where agreement reaches $0.999$ --- so
the last three rows record where the run was stopped, not a ceiling on what the
architecture can express.

In the second experiment, we study the ability of a transformer to track the same
escape path as a tabular model would take. To this end, we follow the
Monte-Carlo chain of one successful run for each of three untied classes at
$T=6$ --- $(0,0)$, $(1,1)$ and $(0,3)$ --- started from a random policy at
$\temp=10^{-6}$ and run until it reaches $\J=1$; no accepted move lowers the
reward. From each chain we take three stretches of $21$ consecutive policies:
its first, its middle, and the last $21$ before $\J=1$, where the reward rises
from chance to the optimum. The transformer is distilled onto each stretch in a
chain, one network per policy, $\theta_{i+1}$ fine-tuned from $\theta_i$, and
every network reproduces its policy exactly before it is used. The barrier of a
step is then the largest amount by which the network's reward falls below the
lower of its two endpoints along the straight line $\theta_i \to \theta_{i+1}$,
in reward quanta $5^{-6}$, evaluated by exhaustive enumeration over all
$5^6 = 15{,}625$ inputs. Results are shown in Table~\ref{tab:escape}, showing
that the embedding adds no barrier on $159$ of the $180$ steps, and none at all
on the last stretch before the optimum: $60$ of $60$ steps on all three classes.
The barriers that do occur are early and small --- at most $35$ quanta, a reward
drop of $2\times10^{-3}$ --- and they are most frequent on the class with the
narrowest window.

\begin{table}[t]
\centering\small
\begin{tabular}{lrr}
\toprule
class & reachable cells & steps to $0.999$ agreement \\
\midrule
$(0,0)$ untied $T{=}4$ & $80$        &  $\le 500$ \\
$(0,0)$ untied $T{=}6$ & $130$       &  $\le 500$ \\
$(0,0)$ untied $T{=}7$ & $155$       &  $\le 500$ \\
$(0,0)$ untied $T{=}8$ & $180$       &  $\le 500$ \\
$(1,1)$ untied $T{=}7$ & $3{,}275$   &  $2{,}500$ \\
$(0,3)$ \textbf{tied} $T{=}7$ & $4{,}525$ & $12{,}500$ \\
$(0,3)$ untied $T{=}7$ & $10{,}775$  &  $\mathbf{47{,}000}$ \\
$(2,1)$ untied $T{=}7$ & $13{,}775$  &  $\mathbf{37{,}000}$ \\
\bottomrule
\end{tabular}
\caption{Eight random untied tabular policies used as teachers for dense
(cross-entropy reward on each CoT token) transformer training, three seeds each.
Transformers express better and train quicker on narrower tabular models.
Agreement is evaluated every $500$ steps, so the step counts are multiples of
that and $\le 500$ is a bound rather than a measurement.}
\label{tab:containment}
\end{table}

\begin{table}[t]
\centering\small
\begin{tabular}{llrrr}
\toprule
class & stretch & steps & barrier $=0$ & largest barrier (in quanta=$5^{-6}$) \\
\midrule
$(0,0)$ untied & first & $20$ & $13$ & $10$ \\
               & middle & $20$ & $15$ & $35$ \\
               & \textbf{last} & $20$ & $\mathbf{20}$ & $\mathbf{0}$ \\
$(1,1)$ untied & first & $20$ & $15$ & $2$ \\
               & middle & $20$ & $17$ & $14$ \\
               & \textbf{last} & $20$ & $\mathbf{20}$ & $\mathbf{0}$ \\
$(0,3)$ untied & first & $20$ & $19$ & $1$ \\
               & middle & $20$ & $\mathbf{20}$ & $\mathbf{0}$ \\
               & \textbf{last} & $20$ & $\mathbf{20}$ & $\mathbf{0}$ \\
\bottomrule
\end{tabular}
\caption{Steps of a tabular escape path along which the transformer's reward does
not fall. Each stretch is $21$ consecutive policies of one Monte-Carlo chain,
distilled in a chain and joined by straight lines in weight space; a step counts
as barrier-free when the reward along that line never falls below its lower
endpoint, and the barrier is measured in reward quanta $5^{-6}$. The last stretch, where the chain climbs to the optimum, is
barrier-free on every step of every class.}
\label{tab:escape}
\end{table}


\subsection{Regulator effects}
\label{app:cures}

Here we study in detail the effects of regulator choices for tabular and
transformer RLVR, starting from entropy regulators. For several tabular models on
untied and uncorrelated classes we used uniform entropy regulators --- the mean-field relation is   $\pol_\cell(a)\propto\exp(\dc\Qc/\temp)$ using which we performed $500$ to $3{,}000$ CAVI sweeps on $12$ to
$240$ seeds for each of $10$ to $18$ temperatures. Every run starts from a random
policy, each sweep moves a random $30\,\%$ of the live cells $30\,\%$ of the way
to the stationary form, $\temp$ is held constant with no quench, and we recorded
the greedy/argmax rewards of the trained models, evaluated exactly rather than
sampled. Average behaviour across seeds is shown in
Fig.~\ref{fig:temperature_untied} of the main text, where, at least for models
trained on a random draw of $T$ values, high regulators (close to the collapse
range) are crucial to get perfect rewards: the three multi-length classes peak at
$\temp$ between $1.8\times10^{-5}$ and $5.6\times10^{-5}$ and have lost the gain
again by $10^{-4}$. A more detailed data per seed is given in
Table~\ref{tab:regulator}, which reports, beside the means, how many individual
seeds end at reward exactly $1$; Notably the $\tau$ grid scan here 
runs up through the melt. 

We further studied the effects of random resets, where $2\,\%$ of the reachable
tabular cells are re-drawn from a fresh random point on the simplex every $10$
sweeps. This is applied for $100$ cycles and the iteration is then run
undisturbed to $2{,}000$ sweeps in total ($3{,}000$ on the $(0,3)$ class), with
the unperturbed baseline given the same budget, the same seeds and the same
temperature. As also
shown in Table~\ref{tab:regulator}, this tends to have a more decisive effect
than the regulator: resets take six of the seven classes to the optimum in all or
all but one seed, including the three on which no temperature does anything,
while the regulator reaches it on three. The one class neither repairs is the
long forward window $(0,3)$, where resets recover about two thirds of the deficit
($0.764 \to 0.940$) and no seed reaches $1$.

Next we turn to transformers. Here we train the same four-head one-block
architecture used throughout the work (width $256$, rotary position encoding,
$791{,}296$ parameters) with REINFORCE and a batch-mean baseline, batch $512$,
learning rate $3\times10^{-4}$, no exploration mixture, $4\times10^{4}$ steps,
under the gated curriculum: training starts at length $1$ and the length
advances when the batch-mean accuracy passes $0.90$, up to a cap of $21$. At every current maximal length, we draw samples uniformly over all lengths up to the current one. The regulator is an entropy
bonus $\beta$ in the occupancy convention, which is what a sampled
policy-gradient trainer implements and which does not convert to the tabular
$\temp$ by any single constant (\S\ref{app:cures}). Accuracy is read at length
$8$, on $2{,}000$ held-out inputs, by greedy decoding at the final step; eight
seeds at $\beta^{-1} = 0$ and $\beta^{-1} = 0.3$, five or six at the other values. As shown
in Fig.~\ref{fig:temperature_untied} at $\beta^{-1} \le 0.1$ every seed of all three tasks
is at chance, and for $\mathbb{Z}_5$ and the quasigroup the whole effect
appears between $0.1$ and $0.3$. At the
working point the seed spread is large --- $\mathbb{Z}_5$ $0.871 \pm 0.266$ with
five of eight seeds above $0.95$, the quasigroup $0.770 \pm 0.282$ with three of
eight, parity at chance --- and the value that works is not shared: $\beta = 1$
carries $\mathbb{Z}_5$ to $0.992$ in five of five seeds and parity to its only
successes ($2/5$), while the quasigroup collapses back to $0.216$.

Carrying over the resets of cells from the tabular case is ill-defined here:
there is no table cell to re-draw, and the operations on the weights that might
stand in for one change something other than which cells receive traffic.
Instead, we studied how relabelling of the emitted token affects performance: a
permutation $\rho$ of the five output symbols is drawn once per rollout and held
fixed for that rollout, the network is fed $\rho(y_t)$ rather than $y_t$, and the
update uses the log-probability of the action actually executed. This is
correlated exploration noise, and it is the transformer counterpart of a reset in
that it changes which cells are consulted while leaving their content alone.
Over eighteen settings (transposition and full permutation, applied to a fraction
of the batch, on steps that are multiples of $|\cal{C}|$, eight seeds each, at
$\beta^{-1} = 0$, we found it to be mildly beneficial: the best setting --- a full
permutation once every $500$ steps --- gives $0.542 \pm 0.081$ against the
$0.417 \pm 0.101$ of the unperturbed run at $T=2$ when training on $T=\{1,2\}$ (where $T=1$ the correct answer is to copy the single input token), and none of the eight
seeds solve the task. Training on $T=\{1,2,3\}$ shows a minor positive signal within seed noise level. 

\begin{table}[t]
\centering\footnotesize
\setlength{\tabcolsep}{3pt}
\begin{tabular}{lrrlrrlrrr}
\toprule
 & & & & \multicolumn{2}{c}{$\temp \to 0$} & \multicolumn{3}{c}{best $\temp$ of the scan} & J @ $\tau_{max}$ \\
\cmidrule(lr){5-6}\cmidrule(lr){7-9}
class & seeds & sweeps & range of $\temp$ & $\J$ & at $1$ & $\temp^\ast$ & $\J$ & at $1$ &  \\
\midrule
parity $(0,0)$, $T{=}10$ & $100$ & $2000$ & $[10^{-15},10^{-2}]$ & $1.0000$ & $100/100$ & \emph{none} & --- & --- & $0.5000$ \\
$\mathbb{Z}_5$ $(0,0)$, $T{=}8$ & $240$ & $3000$ & $[10^{-12},10^{-7}]$ & $0.7995$ & $84/240$ & \emph{none} & --- & --- & $0.2061$ \\
quasigroup $(0,0)$, $T{=}8$ & $240$ & $3000$ & $[10^{-12},10^{-7}]$ & $0.7858$ & $67/240$ & $10^{-10}$ & $0.7992$ & $75/240$ & $0.2000$ \\
quasigroup $(2,1)$, $T{\in}\{3,4,5\}$ & $12$ & $500$ & $[10^{-13},10^{-2}]$ & $0.9837$ & $0/12$ & $3.2{\times}10^{-5}$ & $\mathbf{1.0000}$ & $\mathbf{12/12}$ & $0.9539$ \\
quasigroup $(1,1)$, $T{\in}\{3..6\}$ & $12$ & $500$ & $[10^{-13},10^{-2}]$ & $0.9493$ & $0/12$ & $5.6{\times}10^{-5}$ & $\mathbf{0.9993}$ & $\mathbf{11/12}$ & $0.8851$ \\
quasigroup $(2,1)$, $T{\in}\{4,5,6\}$ & $12$ & $500$ & $[10^{-13},10^{-2}]$ & $0.9626$ & $0/12$ & $1.8{\times}10^{-5}$ & $\mathbf{0.9998}$ & $\mathbf{9/12}$ & $0.9001$ \\
quasigroup $(0,3)$, $T{=}7$ & $12$ & $3000$ & $[10^{-13},10^{-7}]$ & $0.7636$ & $0/12$ & \emph{none} & --- & --- & $0.1999$ \\
\midrule
\multicolumn{10}{l}{\emph{Re-randomisation, against a baseline at the same temperature, seeds and budget}}\\
\midrule
 & & & & \multicolumn{2}{c}{plain} & \multicolumn{3}{c}{$+$ re-randomisation} & \\
\cmidrule(lr){5-6}\cmidrule(lr){7-9}
class & seeds & sweeps & $\temp$ & $\J$ & at $1$ & & $\J$ & at $1$ & \\
\midrule
parity $(0,0)$, $T{=}10$ & $100$ & $2000$ & $10^{-13}$ & $1.0000$ & $100/100$ & & $1.0000$ & $100/100$ & \\
$\mathbb{Z}_5$ $(0,0)$, $T{=}8$ & $60$ & $2000$ & $10^{-10}$ & $0.7770$ & $17/60$ & & $\mathbf{0.9967}$ & $\mathbf{59/60}$ & \\
quasigroup $(0,0)$, $T{=}8$ & $60$ & $2000$ & $10^{-10}$ & $0.8190$ & $23/60$ & & $\mathbf{1.0000}$ & $\mathbf{60/60}$ & \\
quasigroup $(2,1)$, $T{\in}\{3,4,5\}$ & $12$ & $2000$ & $0$ & $0.9837$ & $0/12$ & & $\mathbf{1.0000}$ & $\mathbf{12/12}$ & \\
quasigroup $(1,1)$, $T{\in}\{3..6\}$ & $12$ & $2000$ & $0$ & $0.9493$ & $0/12$ & & $\mathbf{1.0000}$ & $\mathbf{12/12}$ & \\
quasigroup $(2,1)$, $T{\in}\{4,5,6\}$ & $12$ & $2000$ & $0$ & $0.9626$ & $0/12$ & & $\mathbf{1.0000}$ & $\mathbf{12/12}$ & \\
quasigroup $(0,3)$, $T{=}7$ & $12$ & $3000$ & $10^{-13}$ & $0.7636$ & $0/12$ & & $0.9400$ & $0/12$ & \\
\bottomrule
\end{tabular}
\caption{The two regulators on the untied classes with i.i.d.\ inputs, per seed.
All entries are greedy rewards of the endpoint code; \emph{at $1$} counts the
seeds ending at reward exactly $1$ out of the seeds run, which is the informative
statistic because these distributions are bimodal. \emph{Upper block, the
temperature scan:} $\temp \to 0$ is the coldest point of the scan and \emph{top}
its warmest, so the two outer columns bracket the range; \emph{none} means no
temperature in the scan beats the coldest point, and on the $\mathbb{Z}_5$ row the
best mean ($0.8030$ at $\temp=10^{-8}$) is within noise of the coldest and its
seed count is lower ($82/240$), so it is reported as no effect. Standard
deviations are omitted for width and are $\le 0.013$ on the narrow-window rows,
$\le 0.024$ on $(0,3)$ and $\approx 0.16$--$0.17$ on the two $(0,0)$ rows at
$T{=}8$. The scan is refined to half and quarter decades near each peak; a grid
of one point per decade puts the $(2,1)\{3,4,5\}$ optimum at $10^{-5}$ and
$0.9960$ rather than at $3.2\times10^{-5}$ and $1$. \emph{Lower block,
re-randomisation:} $2\,\%$ of the reachable cells re-drawn every $10$ sweeps for
$100$ cycles, then undisturbed to the stated budget; the plain column is the same
protocol with the perturbation switched off, at the same temperature, seeds and
budget, so the pair is one controlled comparison. The three narrow-window rows
were run with the regulator off ($\temp = 0$) rather than at the cold end of the
scan, and give the same plain values as the scan does to four decimals.}
\label{tab:regulator}
\end{table}


\subsection{Correlations and tying induce local reward maxima}
\label{app:trapsexist}

Here we discuss two separate mechanisms which trap spin and RLVR models. As a reminder, we say a policy is \emph{trapped}
if no single-cell change raises the reward \emph{and} no path of reward-neutral
single-cell changes reaches a policy from which one does. 

One way of inducing such traps comes from the shortcut learning perspective: when
a simpler heuristic is available for solving the task, the model may commit to a
heuristic strategy from which any change towards a better generalizing solution is
locally harmful. To induce this option for our task, we introduce input
correlation by drawing inputs from a Markov chain having a $k$-chance of emitting
the same token and a random chance for all other tokens. Notably, correlations
also invalidate elements of our proof in App.~\ref{app:escape}. Similarly,
introducing tied policies invalidates our proof, and so we consider those as well.

This appendix establishes numerically that both these augmentations have a finite
chance of introducing traps. With correlated inputs, the models are small enough to
settle the question by complete enumeration: at $T\le 8$ the policy space has at
most $2^{30} = 1.1\times10^{9}$ members, so every policy is visited, and every
local maximum can be classified. The entire set of traps is
therefore known rather than sampled. Traps are rare --- above $99.99\,\%$ of local
maxima have a cost-free exit. Notwithstanding, the Monte-Carlo chain reaches a trapped state in 
about one run out of eight, evidence that traps carry large basins. Every trap
turns out to be one of a small family of shortcut policies that discard one or two
input positions. That family is generated by the correlation, and both of its reward 
spectra are geometric in $2k-1$: a policy that discards $m$ input positions scores
$\tfrac{1}{2}(1+(2k-1)^m)$, which at $k=0.9$ gives $0.90$ and $0.82$ for $m=1,2$,
and the escape barriers take the values $0.04\times(2k-1)^{j}$. 

Next we consider tied tabular policies with uncorrelated inputs. For the $(0,3)$ tied
class at $T=7$, full enumeration of codes is too costly. Instead, we ran many Monte Carlo chains till saturation (see below). For each chain endpoint we performed a search over its
zero-loss neighbourhood. The certification procedure is the following. Call the
\emph{shelf} of a policy the connected component containing it of the graph whose
nodes are policies and whose edges are single-cell changes with exactly zero
reward difference, and call a shelf node an \emph{exit} if it admits a
single-cell change that strictly raises the reward. Starting from the endpoint we
search its shelf breadth-first, evaluating at every node all 
single-cell changes (of which there are $4{,}525\times 4$), which both tests for an exit and generates the next set of neutral
edges. We comment that all reward
differences are exact re-simulations rather than the first-order field, which is
exact for untied tables but not under tying. The search is capped at $10^{5}$
shelf nodes. A policy whose shelf is exhausted with no exit found is
\emph{certified}: no chain of non-losing single-cell changes ever improves it. One
that reaches the cap first is reported as \emph{bounded}, a weaker statement.

Following this procedure, we ran Monte-Carlo on twelve different initializations at
temperature $\temp=10^{-6}$ for $20k$ sweeps. None of the twelve reaches
the optimum. They freeze at twelve distinct rewards between $0.830$ and $0.976$,
mean $0.897 \pm 0.044$. In every one of the twelve the reward is unchanged
from $5\times10^{3}$ sweeps to $2\times10^{4}$ --- identical to machine
precision, not merely non-improving. Ten of the twelve are certified as above, and two reach the maximal node
count and are reported as bounded. The shelves are small and simply structured ---
between $1$ and $768$ nodes, of diameter at most $9$, with sizes that factor as
products of small integers, so neutral moves do not unlock further ones --- and
two of the certified endpoints have a shelf of a single node, meaning no
single-cell change of any kind, improving or flat, exists at rewards of $0.976$
and $0.938$.

\begin{table}[t]
\centering\footnotesize\setlength{\tabcolsep}{4pt}
\begin{tabular}{rrrrrrl}
\toprule
\multicolumn{7}{l}{\emph{Correlated inputs, $k=0.9$: every policy visited, every local maximum classified}}\\
\midrule
$T$ & cells & policies & local maxima & \textbf{traps} & trap frac. & escape barriers \\
\midrule
$4$ & $14$ & $1.6{\times}10^{4}$ & $648$        & $\mathbf{8}$      & $1.2{\times}10^{-2}$ & $0.040$ \\
$5$ & $18$ & $2.6{\times}10^{5}$ & $14{,}656$   & $\mathbf{32}$     & $2.2{\times}10^{-3}$ & $0.040$ \\
$6$ & $22$ & $4.2{\times}10^{6}$ & $141{,}280$  & $\mathbf{160}$    & $1.1{\times}10^{-3}$ & $0.040$, $0.032$ \\
$7$ & $26$ & $6.7{\times}10^{7}$ & $3{,}622{,}336$ & $\mathbf{576}$ & $1.6{\times}10^{-4}$ & $0.040$, $0.032$ \\
$8$ & $30$ & $1.1{\times}10^{9}$ & $35{,}307{,}392$ & $\mathbf{2{,}432}$ & $6.9{\times}10^{-5}$ & $0.040$, $0.032$, $0.0256$ \\
\midrule
\multicolumn{7}{l}{\emph{Shared table, $(0,3)$ tied $T{=}7$, $4{,}525$ cells: certificates, one per endpoint}}\\
\midrule
\multicolumn{4}{l}{endpoints examined, from independent chain runs} & \multicolumn{3}{l}{$12$} \\
\multicolumn{4}{l}{certified: shelf exhausted, no improving move} & \multicolumn{3}{l}{$\mathbf{10}$} \\
\multicolumn{4}{l}{bounded: search cap reached before exhaustion} & \multicolumn{3}{l}{$2$} \\
\bottomrule
\end{tabular}
\caption{Two modifications that put traps into an otherwise featureless landscape.
A \emph{trap} is a policy with no improving single-cell change and no cost-free
path to one. \emph{Top:} with inputs from a Markov chain of persistence $k$, every
policy in the space is visited and every local maximum classified, so these counts
are complete rather than sampled. Traps are rare, yet the chain lands in one about
once in eight runs, because they carry the large basins. Removing the correlation
removes them entirely. \emph{Bottom:} with one table shared across positions the
space cannot be enumerated, so each endpoint is certified individually by
exhausting its shelf of reward-neutral changes and testing every neighbour of
every shelf node for improvement.}
\label{tab:trapsexist}
\end{table}

\end{document}